\documentclass[journal]{IEEEtran}
\usepackage[utf8]{inputenc} % allow utf-8 input
\usepackage[T1]{fontenc}    % use 8-bit T1 fonts
\usepackage{hyperref}       % hyperlinks
\usepackage{url}            % simple URL typesetting
\usepackage{booktabs}       % professional-quality tables
\usepackage{amsfonts}       % blackboard math symbols
\usepackage{xcolor}         % colors

\usepackage{subfigure}
\usepackage{framed,multirow}
\usepackage{amssymb}
\usepackage{latexsym}
\usepackage{colortbl}
\usepackage{makecell}
\usepackage{comment}
\usepackage{algorithm}
\usepackage{algorithmic}
\usepackage{eqparbox}
\usepackage{amsmath}
\usepackage{amsthm}
\usepackage{graphicx}
\usepackage{color}
\usepackage{array}
\usepackage{bm}
\usepackage[nolist,nohyperlinks]{acronym}

\newtheorem{definition}{Definition}
\newtheorem{proposition}{Proposition}
\newtheorem{claim}{Claim}

\acrodef{DL}{Deep Learning}
\acrodef{FL}{Federated Learning}
\acrodef{DNN}{Deep Neural Network}
\acrodef{GDPR}{General Data Protection Regulation}
\acrodef{DLG}{Deep Leakage from Gradients}
\acrodef{iDLG}{Improved \ac{DLG}}
\acrodef{IG}{Inverting Gradient}
\acrodef{GGL}{Generative Gradient Leakage}
\acrodef{PRECODE}{Privacy Enhancing Module}
\acrodef{CMA-ES}{Covariance Matrix Adaptation Evolution Strategy}
\acrodef{GRNN}{Generative Regression Neural Network}
\acrodef{GAN}{Generative Adversarial Network}
\acrodef{CNN}{Convolutional Neural Network}
\acrodef{DP}{Differential Privacy}
\acrodef{HE}{Homomorphic Encryption}
\acrodef{MPC}{Secure Multi-Party Computation}
\acrodef{MSE}{Mean Square Error}
\acrodef{PSNR}{Peak Signal-to-Noise Ratio}
\acrodef{LPIPS}{Learned Perceptual Image Patch Similarity}
\acrodef{SSIM}{Structure Similarity Index Measure}
\acrodef{ReLU}{Rectified Linear Unit}
\acrodef{CE}{Cross Entropy}
\acrodef{BN}{Batch Normalization}
\acrodef{FedAvg}{Federated Averaging}
\acrodef{SGD}{Stochastic Gradient Descent}

\begin{document}

\title{Gradient Leakage Defense with Key-Lock Module for Federated Learning}

\author{Hanchi~Ren,~
        Jingjing~Deng,~
        Xianghua~Xie*~\IEEEmembership{IEEE Senior Member},\\
        % Xiaoke Ma,
        % and~Jianfeng Ma~\IEEEmembership{IEEE Member}% <-this % stops a space
\IEEEcompsocitemizethanks{\IEEEcompsocthanksitem Hanchi Ren and Xianghua Xie are with the Department of Computer Science, Swansea University, United Kingdom. Jingjing Deng is with the Department of Computer Science, Durham University, United Kingdom. 
% Xiaoke Ma is with the School of Computer Science and Technology, and Jianfeng Ma is with School of Cyber Engineering, Xidian University, China. The corresponding author is Xianghua Xie.
\protect\\
E-mail:\{hanchi.ren, x.xie\}@swansea.ac.uk, and jingjing.deng@durham.ac.uk
% , and \{xkma, jfma\}@mail.xidian.edu.cn
}% <-this % stops an unwanted space
% \thanks{Manuscript received April 19, 2005; revised August 26, 2015.}
}

\maketitle

\begin{abstract}
  Federated Learning (FL) is a widely adopted privacy-preserving machine learning approach where private data remains local, enabling secure computations and the exchange of local model gradients between local clients and third-party parameter servers. However, recent findings reveal that privacy may be compromised and sensitive information potentially recovered from shared gradients. In this study, we offer detailed analysis and a novel perspective on understanding the gradient leakage problem. These theoretical works lead to a new gradient leakage defense technique that secures arbitrary model architectures using a private key-lock module. Only the locked gradient is transmitted to the parameter server for global model aggregation. Our proposed learning method is resistant to gradient leakage attacks, and the key-lock module is designed and trained to ensure that, without the private information of the key-lock module: a) reconstructing private training data from the shared gradient is infeasible; and b) the global model's inference performance is significantly compromised. We discuss the theoretical underpinnings of why gradients can leak private information and provide theoretical proof of our method's effectiveness. We conducted extensive empirical evaluations with many models on several popular benchmarks, demonstrating the robustness of our proposed approach in both maintaining model performance and defending against gradient leakage attacks.
\end{abstract}


\IEEEpeerreviewmaketitle

%=====================================================================
\section{Introduction}

\IEEEPARstart{D}{eep} Neural Networks are data-intensive, and growing public concerns about data privacy and personal information have garnered significant attention from researchers in recent years. \ac{FL}, a decentralized framework for collaborative privacy-preserving model training, was proposed to address the data privacy issue~\cite{mcmahan2016communication, konevcny2016federated, konevcny2016federatedlearning, mcmahan2017federated,lu2022personalized}. This framework involves multiple training clients and a central server responsible for aggregating locally computed model gradients into a global model for all clients to share. Sensitive data remains exclusively accessible to its respective owner. Nonetheless, recent studies have shown that gradient-sharing schemes do not adequately protect sensitive and private data. 

Regarding gradient leakage attack, it describes techniques that utilize gradients from a target model to uncover privacy-sensitive information. \ac{DL} models are trained on datasets by updating parameters to align with the feature space, creating a close relationship between the gradients and the dataset. Consequently, numerous studies focus on exploiting these gradients to reveal private information. Gradient leakage approaches have proven to be highly effective and successful. Notably, gradient leakage can occur even in models that have not yet converged~\cite{zhu2019deep}.

\begin{figure}[t!]
    \centering
    \includegraphics[width=0.99\linewidth]{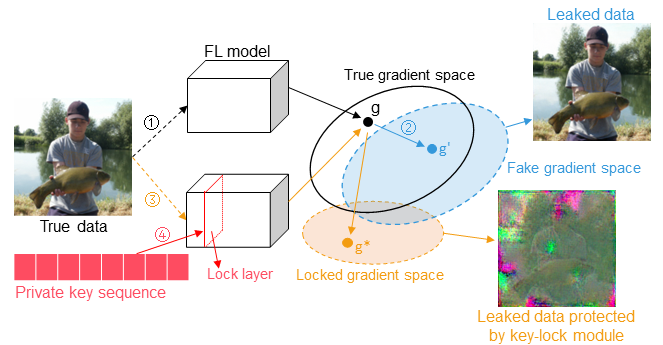}
    \caption{Illustration of FedKL against gradient leakage: 1. Clients feed private training data into the model and generate the true gradient; 2. The attacker reconstructs private images from the shared gradient; 3. Clients feed private training data into the model with key-lock module. The attacker can not reveal true images from the locked gradient; 4. A private key sequence is fed into the lock layer for proper inference progression.}
    \label{fig:illustration}
    % \vspace{-8pt}
\end{figure}

Recently, a number of strategies have been proposed in order to mitigate or overcome leakage attacks, including gradient perturbation~\cite{zhu2019deep, yang2020accuracy, sun2020ldp, sun2021soteria, ren2022grnn, fan2024guardian, zhu2024vulnerabilities}, data obfuscation or sanitization~\cite{hasan2016effective, chamikara2018efficient, chamikara2020efficient, lee2021digestive, chamikara2021privacy, scheliga2022precode, wu2024concealing}, and hybrid methods~\cite{bu2020deep, ren2020fedboost, li2020privacy, yadav2020differential, wei2021gradient}. Nonetheless, in the majority cases, a balance must be struck between privacy preservation and performance optimization. Specifically, for conventional encryption-based methodologies, the elevated computational complexity intrinsic to cryptographic operations demands substantial computational resources. An in-depth examination of gradient leakage ablation experiments was conducted by Wei \emph{et al.}~\cite{wei2020framework}, in which the study quantitatively delineated the relationships between \ac{DNN} architectural design and \ac{FL} settings, encompassing factors such as batch size, image resolution, activation function, and the quantity of local iterations preceding gradient exchange. These findings are either directly or indirectly corroborated by several other studies. For instance, \ac{DLG} proposes that the activation function must be twice-differentiable. The research also assesses the impact of leakage performance by altering the number of local training iterations on the client side. In~\cite{fan2021deepip}, the authors introduced a neural network ownership verification technique, DeepIP, a concept referred to as a ``digital passport", which serves as input for the ``passport function" responsible for generating the scale and shift parameters in the normalization layer. A key finding of this study is that the model's inference performance will experience significant degradation in the presence of forged passports. In our proposed work, we show that transitioning the scale and shift parameters from trainable parameters to generated parameters has the potential to disrupt the information inheritance process during backward propagation.

In this study, we first establish, from a theoretical standpoint, that the feature maps computed from the fully-connected layer, convolutional layer, and \ac{BN} layer encompass the private information of input data. This information also coexists in the gradient during the backward propagation. Moreover, we propose that gradient leakage attacks are only viable when the gradient spaces between the global model and local models are adequately aligned. Consequently, we introduce FedKL, a key-lock module adept at differentiating, misaligning, and securing the gradient spaces with a private key while maintaining federated aggregation analogous to the standard \ac{FL} framework. In essence, we reconfigure the scale and shift procedures in the normalization layer. A private key, \emph{i.e.}, a randomly generated sequence, is input into two fully-connected layers, and the resultant outputs serve as the proprietary coefficients for the scale and shift processes. Both theoretical analyses and experimental outcomes demonstrate that the proposed key-lock module is effective and efficient in defending against gradient leakage attacks, as the consistency of private information in the gradient is obscured. Consequently, a malicious attacker is unable to formulate forward-backward propagation without the private key and the gradient of the lock layer. Thus, it becomes infeasible to reconstruct local training data by approximating the shared gradient in the \ac{FL} system. Our theoretical and experimental findings indicate that FedKL offers the following advantages:
\begin{itemize}
    \item robust protection against gradient leakage attacks;
    \item negligible decline in inference performance compared to the model without the key-lock module;
    \item minimal additional computational cost associated with the key-lock module;
    \item significant reduction in inference performance when the private key of the key-lock module is absent;
    \item ability to adapt to arbitrary network architecture.
\end{itemize}

The remainder of this paper is structured as follows: Section~\ref{sec:rw} encompasses a discussion of relevant literature on gradient leakage attacks and associated defensive strategies; Section~\ref{sec:pm} introduces the proposed gradient leakage defense approach and private key-lock module, succeeded by an experimental evaluation and comparative analysis in Section~\ref{sec:er}; Finally, Section~\ref{sec:cc} offers concluding remarks and explores potential avenues for future research.

%=====================================================================
\section{Related Work}
\label{sec:rw}

Gradient leakage attacks seek to disclose private training data based on leaked gradients and represent a significant privacy-preserving threat to both centralized and collaborative \ac{DL} systems. In centralized \ac{DL} systems, membership inference is a common method for exposing private training data~\cite{shokri2017membership, truex2018towards, truex2019demystifying, choquette2021label, zhang2022label}. Zhu \emph{et al.}~\cite{zhu2019deep} proposed the \ac{DLG} method for recovering original training data by iteratively updating a randomly initialized input image and its associated label to approximate the true gradient. The updated input image is expected to converge towards the true image when the training is complete. However, \ac{DLG} has been criticized for its instability and sensitivity to batch size and image resolution. It was one of the first studies to investigate private training data reconstruction from leaked gradients in a collaborative \ac{DL} system. Zhao \emph{et al.}~\cite{zhao2020idlg} demonstrated that the ground-truth label can be computed analytically from the gradient of the loss with respect to the \emph{Softmax} layer and proposed \ac{iDLG}. Following \ac{DLG}, Geiping \emph{et al.}~\cite{geiping2020inverting} explored an alternative approach, termed \ac{IG}, to unveiling private training data using an optimization-based method. They enhanced the stability of \ac{DLG} by introducing a magnitude-invariant cosine similarity measurement for the loss function and demonstrated that prior knowledge facilitates a more efficient gradient leakage attack. \ac{IG} can recover images up to $224*224$ resolution and a batch size of $100$ with acceptable success rates. The following work from Wen \emph{et al.}~\cite{wen2022fishing} proposed a data fishing method, relying on optimization-based methods and gradient magnification strategies. It primarily focuses on recovering images from large batch sizes by magnifying the gradient contributions of specific data points or classes. However, data fishing requires multiple rounds of FL training for specific gradient and class magnification and can only recover one class of images at a time. The image reconstruction process of data fishing is similar to \ac{IG}. Similarly, Jeon \emph{et al.}~\cite{jeon2021gradient} posited that gradient information alone is insufficient for revealing private training data. As a result, they proposed GIAS, which employs a pre-trained data-revealing model. Yin \emph{et al.}~\cite{yin2021see} reported that in an image classification task, the ground-truth label could be easily discerned from the gradient in the last fully-connected layer, and \ac{BN} statistics can significantly augment the gradient leakage attack, ultimately revealing high-resolution private training images.

Generative model-based methods can also facilitate gradient leakage attacks. Hitaj \emph{et al.}~\cite{hitaj2017deep} proposed a \ac{GAN}-based data recovery method capable of generating a dataset with a distribution closely resembling that of the training dataset. Wang \emph{et al.}~\cite{wang2019beyond} trained a \ac{GAN} with a multitask discriminator for auxiliary identification, referred to as mGAN-AI, to generate private information based on the gradient. In our previous work~\cite{ren2022grnn}, we introduced a new method for recovering training data using a \ac{GRNN}. \ac{GRNN} consists of two generative model branches, one based on a \ac{GAN} for generating fake training data and the other using a fully-connected architecture for generating corresponding labels. The training data is recovered by regressing the true and fake gradients, computed from the generated data pairs. The \ac{GRNN} was developed to reconstruct private training data and its corresponding label, handling large batch sizes and high-resolution images. In the \ac{GGL}~\cite{li2022auditing} study, a \ac{GAN} was employed to create fake data. The weights of the \ac{GAN} were pre-trained and fixed, with the trainable parameters in \ac{GGL} being the input sequence of the \ac{GAN}. The label inference component was derived from \ac{iDLG}, requiring a batch size of 1. In contrast to other approaches, \ac{GGL} utilized \ac{CMA-ES} as the optimizer to reduce variation in the generated data. However, \ac{GGL} only generates similar data given a label, and it does not aim to recover the original training data.

Many attempts have been dedicated to protecting private information leakage from gradients. A significant proportion of these, including gradient perturbation, data obfuscation or sanitization, \ac{DP}, \ac{HE}, and \ac{MPC}~\cite{li2020privacy}, attempt to safeguard private training data or publicly shared gradients transmitted between clients and servers. Zhu \emph{et al.}~\cite{zhu2019deep} experimented with adding Gaussian and Laplacian noise, reporting that only the magnitude of distribution variance matters, not the noise type. Leakage attacks fail when the variance exceeds $10^{-2}$, but at this level of variance, the global model performance dramatically declines. Chamikara \emph{et al.}~\cite{chamikara2021privacy} proposed a data perturbation method that preserves training data privacy without compromising model performance. The input dataset is treated as a data matrix and subjected to a multidimensional transformation in a new feature space. They perturbed the input data with various transformation scales to ensure adequate perturbation. However, the method relies on a centralized server controlling the generation of global perturbation parameters, and crucially, the perturbation method may corrupt the architectural information of image-based data. Wei \emph{et al.}~\cite{wei2021gradient} employed \ac{DP} to add noise to each client's training dataset and introduced a per-example-based \ac{DP} approach for the client, called \emph{Fed-CDP}. To improve inference performance and defend against gradient leakage, a dynamic decay noise injection method is proposed. However, experimental results reveal that although the method prevents training data reconstruction from the gradient, inference accuracy significantly deteriorates. Moreover, the computational cost is high since \ac{DP} is performed on each individual training sample. Both \ac{PRECODE}~\cite{scheliga2022precode} and our proposed FedKL attempt to prevent input information from passing through the model when computing the gradient. \ac{PRECODE} uses variational modeling to introduce stochastic sampling in the latent feature space of neural networks. By making latent features stochastic, the gradient information sent during collaborative training does not directly correlate to the input data. \ac{PRECODE} integrates into two fully-connected layers as additional layers before the final classification layer, ensuring that the output features for prediction are derived from the stochastic latent representation. While our method, FedKL, and \ac{PRECODE} share the goal of blocking private information transmission and disrupting the optimization process of gradient inversion attacks, there are significant differences between the two approaches. FedKL specifically targets the \ac{BN} layer, maintaining normal optimization in other layers. Since the parameters of the \ac{BN} layer are only related to data distribution, our method has a smaller performance impact on the model. In contrast, \ac{PRECODE} introduces stochastic noise into the latent feature space, which can affect model performance and convergence. By focusing on the \ac{BN} layer, FedKL offers a novel and efficient way to protect privacy in \ac{FL} systems with minimal impact on model performance. This distinction sets our work apart from existing methods like \ac{PRECODE} and highlights the unique contributions of our approach.

%=====================================================================
\section{Proposed Method}
\label{sec:pm}

In this section, we first demonstrate that the gradient information of a \ac{DNN} during the back-propagation stage is strongly correlated with the input data and its corresponding label in a supervised learning context. Subsequently, we theoretically illustrate that by shifting the latent space, the private data cannot be recovered from the gradient without aligning the latent spaces. Based on these two theoretical discoveries, the key-lock module is proposed to safeguard the local gradient from leakage attacks. Finally, we present the novel \ac{FL} framework with the key-lock module, referred to as FedKL, which illustrates how the proposed method protects the collaborative learning system from gradient leakage attacks.

%----------------------------------------------------------------------
\subsection{Theoretical Analysis on Gradient Leakage}
\label{sec:TAGL}
Recently, several studies have demonstrated that the input data information embedded in the gradient enables leakage attacks. Geiping \emph{et al.}~\cite{geiping2020inverting} established that for any fully-connected layer in a neural network, the derivative of the loss value concerning the layer's output contains input data information. By employing the \emph{Chain Rule}, the input of the fully-connected layer can be computed independently from the gradients of other layers. In this section, using typical supervised learning tasks, \emph{e.g.} classification and regression, we extend this theory to the convolutional layer and \ac{BN} layer.

\begin{definition}
    \textbf{Gradient:} suppose in a vector or matrix space, a function $f: \mathbb{R}^{n} \rightarrow \mathbb{R}^{m}$ that maps a vector of length $n$ to a vector with length $m$: $f(x) = y,\ x \in \mathbb{R}^{n};\ y \in \mathbb{R}^{m}$. The gradient is the derivative of a function $f$ with respect to input $x$, and can be presented by the Jacobian Matrix:
    \begin{gather}
    \label{alg:jacobian-f}
    \frac{\partial f}{\partial x} = 
        \begin{pmatrix}
            \frac{\partial f_1}{\partial x_1} & \cdots & \frac{\partial f_1}{\partial x_n} \\
            \vdots & \ddots & \vdots \\
            \frac{\partial f_m}{\partial x_1} & \cdots & \frac{\partial f_m}{\partial x_n}
        \end{pmatrix}
    \end{gather}
\end{definition}

\begin{definition}
    \textbf{Chain Rule:} considering the derivative of a function of a function, \emph{i.e.}, $g(f(x))$. Assuming that the function is continuous and differentiable, the outer function is differentiable with respect to the inner function as an independent variable. Subsequently, the inner function is differentiable with respect to its independent variable $x$. According to the \emph{Chain Rule}, the derivative of $x$ with respect to $x$ is:
    \begin{gather}
    \label{alg:chain-rule}
        \frac{\partial g}{\partial x} = \frac{\partial g}{\partial f} \frac{\partial f}{\partial x}
    \end{gather}
\end{definition}

\begin{claim}
\label{claim:1}
    The gradient in a linear neural network demonstrates a strong correlation with both the input data and their associated ground-truth labels.
\end{claim}
\begin{proof}
    Assuming a simple linear regression example, the model is $f(x): \hat{y} = \theta x$, the loss function is \ac{MSE}: $\mathcal{L}(x, y) = \frac{1}{N}\sum (\hat{y} - y)^2 = \frac{1}{N}\sum (\theta x - y)^2$, where $N$ is the number of sample, $y$ is the ground-truth. The derivative of loss function $\mathcal{L}$ with respect to the weight $\theta$ on a batch of $N$ samples is:
    \begin{align}
    \label{equ:gradient_loss_linear_function}
        \frac{\partial \mathcal{L}}{\partial \theta} &= \frac{\partial}{\partial \theta} \frac{1}{N} \sum_i(x_i \cdot \theta - y_i)^2\\
        &= \frac{1}{N} \sum_i\frac{\partial}{\partial \theta}(x_i \cdot \theta - y_i)^2 \nonumber\\
        &= \frac{1}{N} \sum_i 2\cdot(x_i \cdot \theta - y_i) \cdot \frac{\partial (x_i \cdot \theta - y_i)}{\partial \theta} \nonumber\\
        &= \frac{1}{N} \sum_i 2\cdot x_i \cdot (\hat{y} - y_i) \nonumber
    \end{align}
    In this case, the gradient of the model is positively related to the input to the function (\emph{i.e.}, $x$) and the difference between the predicted label ($\hat{y}$) and the ground-truth label ($y$). This relationship implies that the gradient carries information about both the input data and the labels, which can be exploited in a gradient leakage attack.
\end{proof}

\begin{claim}
\label{claim:2}
    In the framework of a non-linear neural network, the gradient exhibits a significant correlation with both the input data and the corresponding true label.
\end{claim}
\begin{proof}
    For a classification task, let's consider a model consisting of two fully-connected layers, a \ac{ReLU} activation layer, a \emph{Softmax} layer, and a \ac{CE} loss function. The forward propagation process of the model can be described as follows:
    \begin{align}
        x &= input \nonumber\\
        u &= \theta \cdot x + b_1 \nonumber\\
        a &= ReLU(u) \nonumber\\
        z &= \lambda \cdot a + b_2 \nonumber\\
        \hat{y} &= softmax(z) \nonumber\\
        \mathcal{L} &= CE(y, \hat{y}) \nonumber
    \end{align}
    where $x \in \mathbb{R}^{D_x\times1}$, $\theta \in \mathbb{R}^{D_o \times D_x}$, $b_1 \in \mathbb{R}^{D_o \times 1}$, $\lambda \in \mathbb{R}^{N_c \times D_o}$, $b_2 \in \mathbb{R}^{N_c \times 1}$, $D_x$ is the dimension of the input space, $D_o$ is the number of hidden nodes in the first fully-connected layer, and $N_c$ is the number of classes. The \ac{ReLU} activation function is widely utilized in \ac{DL}. The derivative of the \ac{ReLU} layer with respect to its input can be expressed as:
    \begin{align}
    \frac{\partial ReLU(x)}{\partial x}&= \left\{\begin{array}{rcl} 
                                                1 & \mbox{if} & x>0 \\
                                                0 & \mbox{if} & otherwise
                                           \end{array}\right. = sgn(ReLU(x)) 
    \end{align}
    % \emph{Sigmoid} is another frequently used activation function. The function of \emph{Sigmoid} is $\sigma(x)=\frac{1}{1+e^{-x}}$. And the derivative of it with respect to its input is:
    % \begin{align}
    % \frac{\partial \sigma(x)}{\partial x} &= \sigma(x)(1-\sigma(x))
    % \end{align}
    The derivative of loss function $\mathcal{L}$ with respect to the input of \emph{Softmax} $z$ is $(\hat{y}-y)^T$. Then the gradient of the output layer is:
    \begin{align}
        \frac{\partial \mathcal{L}}{\partial \lambda} &= \frac{\partial \mathcal{L}}{\partial z} \frac{\partial z}{\partial \lambda} = (\hat{y}-y)^T \cdot a \label{equ:gradient_non_neg}\\
        \nonumber \\
        \frac{\partial \mathcal{L}}{\partial b_2} &= \frac{\partial \mathcal{L}}{\partial z} \frac{\partial z}{\partial b_2} = (\hat{y}-y)^T \label{equ:gradient_non_neg_bias}
    \end{align}
    This substantiates that the gradient of the output layer encompasses information about the ground-truth label, which aligns with the findings in papers~\cite{zhao2020idlg,yin2021see}. In Eqn.~\ref{equ:gradient_non_neg}, $\hat{y}$ represents the predicted probability distribution and $y$ denotes the one-hot distribution of the ground-truth label. Consequently, only the correct label yields a negative result for $(\hat{y} - y)^T$. Furthermore, $a$, feature map as the output of the activation layer, passes the information of input data to the gradient of the output layer. As the value of $\theta$ and $b_1$ are known in typical \ac{FL} system, the input data can be calculated based on Eqns.~\ref{equ:gradient_non_neg}, \ref{equ:gradient_non_neg_bias} and \ref{equ:a_x}:
    \begin{align}
    a = \left\{\begin{array}{ccl}
                \theta \cdot x + b_1 & \mbox{if} & \theta \cdot x + b_1>0 \\
                0 & \mbox{if} & otherwise
                \end{array}\right. \label{equ:a_x}
    \end{align}
    Next, the gradient for the first fully-connected layer:
    \begin{align}
        \frac{\partial \mathcal{L}}{\partial \theta} &= \frac{\partial \mathcal{L}}{\partial z} \frac{\partial z}{\partial a} \frac{\partial a}{\partial u} \frac{\partial u}{\partial \theta} \label{equ:gradient_loss_fc_weight}\\
            &= (\hat{y}-y)^T \cdot \lambda \circ sgn(a) \cdot x \nonumber\\
            &= \left\{\begin{array}{ccl}
                (\hat{y}-y)^T \cdot \lambda \cdot x & \mbox{if} & a>0 \nonumber\\
                0 & \mbox{if} & otherwise
                \end{array}\right. \nonumber\\
        \nonumber\\
        \frac{\partial \mathcal{L}}{\partial b_1} &= \frac{\partial \mathcal{L}}{\partial z} \frac{\partial z}{\partial a} \frac{\partial a}{\partial u} \frac{\partial u}{\partial b_1} \label{equ:gradient_loss_fc_bias}\\
            &= (\hat{y}-y)^T \cdot \lambda \circ sgn(a) \cdot 1 \nonumber\\
            &= \left\{\begin{array}{ccl}
                (\hat{y}-y)^T \cdot \lambda & \mbox{if} & a>0 \nonumber\\
                0 & \mbox{if} & otherwise
                \end{array}\right. \nonumber
    \end{align}
    The weight $\lambda = \lambda^{\prime}-\eta (\hat{y}-y)^T a$, where $\lambda^{\prime}$ represents the weight from the previous iteration and $\eta$ denotes the learning rate. In the \ac{ReLU} layer, specific information is suppressed (set to 0), as seen in Eqn.~\ref{equ:gradient_loss_fc_weight}. Notably, the non-zero gradients comprise three components: 1) the difference between the predicted probability distribution and the one-hot ground-truth label; 2) the weights of the subsequent fully-connected layer $\lambda$; and 3) the input of the current layer, which includes input data and feature maps. Given the gradient of each layer, every element in Eqns.\ref{equ:gradient_loss_fc_weight} and \ref{equ:gradient_loss_fc_bias} is known; thus, the input data $x$ can be reconstructed. It can be concluded that the gradient in a non-linear neural network exhibits a strong correlation with the input data and ground-truth label. In accordance with the \emph{Chain Rule}, the input information is also transmitted throughout the entire network.    
\end{proof}

\begin{claim}
\label{claim:3}
    The gradient in a \ac{CNN} exhibits a strong correlation with both the input data and the ground-truth label. In the convolutional layers, input information and features are propagated through the network, influencing the gradients at each layer. 
\end{claim}
\begin{proof}
    In the context of a classification \ac{CNN} composed of a convolutional layer, a \ac{ReLU} activation layer, a fully-connected layer, a \emph{Softmax} layer, and a \ac{CE} loss function, the forward propagation process of the model can be described as follows:
    \begin{align}
        x &= input \nonumber\\
        u &= \theta \otimes x + b_1 \nonumber\\
        a &= ReLU(u) \nonumber\\
        a^{\prime} &= flatten(a) \nonumber\\
        z &= \lambda \cdot a^{\prime} + b_2 \nonumber\\
        \hat{y} &= softmax(z) \nonumber\\
        \mathcal{L} &= CE(y, \hat{y}) \nonumber
    \end{align}
    The derivatives of $\mathcal{L}$ with respect to $\lambda$ and $b_2$ are consistent with those Claim~\ref{claim:2}. Therefore, our primary focus will be on the gradient of the convolutional layer:
    \begin{align}
    \label{equ:conv_gradient}
        \frac{\partial \mathcal{L}}{\partial \theta} &= \frac{\partial \mathcal{L}}{\partial u} \frac{\partial u}{\partial \theta}
    \end{align}
    $\frac{\partial \mathcal{L}}{\partial u}$ is known in Eqns.~\ref{equ:gradient_loss_fc_weight} \& \ref{equ:gradient_loss_fc_bias}. According to the convolutional function, we have:
    \begin{align}
        u_{(i,j,d)} = \sum_h \sum_w \sum_c \theta_{(d,h,w,c)} \cdot x_{(i+h, j+w, c)} + b_{1d}
    \end{align}
    where $i$ and $j$ represent the coordinates of the output feature map, $d$ denotes the index of the convolutional kernel, $h$ and $w$ are the coordinates of the kernel and $c$ is the index of the input channel. Then, the derivative of one pixel of the output feature map against any one specific parameter is:
    \begin{align}
        \frac{\partial u_{(i,j,d)}}{\partial \theta_{(d, h^{\prime},w^{\prime}, c^{\prime})}} &= \frac{\partial (\sum_h \sum_w \sum_c \theta_{(d,h,w,c)} \cdot x_{(i+h, j+w, c)} + b_{1d})}{\partial \theta_{(d, h^{\prime},w^{\prime}, c^{\prime})}} \\
            &= \frac{\partial \theta_{(d, h^{\prime},w^{\prime}, c^{\prime})} \cdot x_{(i+h^{\prime}, j+w^{\prime}, c^{\prime})}}{\partial \theta_{(d, h^{\prime},w^{\prime}, c^{\prime})}} \nonumber 
            = x_{(i+h^{\prime}, j+w^{\prime}, c^{\prime})} \nonumber
    \end{align}
    All parameters contribute to the computation of the final model output, enabling the transformation of Eqn.\ref{equ:conv_gradient} into the derivative of the loss function with respect to any convolutional layer parameter (see Eqn.\ref{equ:gradient_loss_theta_prime}). It is worth noting that each convolutional kernel is associated with only one output channel, meaning that the first kernel with the input generates the first channel of the output, the second kernel with the input results in the second channel of the output, and so on.
    \begin{align}
        \frac{\partial \mathcal{L}}{\partial \theta_{(d, h^{\prime},w^{\prime}, c^{\prime})}} &= \frac{\partial \mathcal{L}}{\partial u_{(i,j,d)}} \frac{\partial u_{(i,j,d)}}{\partial \theta_{(d, h^{\prime},w^{\prime}, c^{\prime})}} \label{equ:gradient_loss_theta_prime}\\
            &= \frac{\partial \mathcal{L}}{\partial u_{(i,j,d)}} \cdot x_{(i+h^{\prime}, j+w^{\prime}, c^{\prime})} \nonumber
    \end{align}
    Regarding the derivative of the loss function with respect to the bias $b_{1d}$, we have:
    \begin{align}
        \frac{\partial u_{(i,j,d)}}{\partial b_{1d}} &= 1 \\
        \nonumber \\
        \frac{\partial \mathcal{L}}{\partial b_{1d}} &= \sum_i \sum_j \frac{\partial \mathcal{L}}{\partial u_{(i,j,d)}}
    \end{align}
    In conclusion, similar to Claim~\ref{claim:1} and Claim~\ref{claim:2}, referring to Eqns.~\ref{equ:gradient_non_neg}, \ref{equ:gradient_non_neg_bias}, \ref{equ:gradient_loss_fc_weight} \& \ref{equ:gradient_loss_fc_bias}, the gradient in the \ac{CNN} is also highly correlated to the input data and ground-truth label. Our previous work~\cite{ren2022grnn} has demonstrated that higher label inference accuracy leads to better reconstruction performance of the input data. This idea has also been presented in papers~\cite{zhao2020idlg,yin2021see}, but the fundamental rationale for such phenomena was not explored.
\end{proof}

It has been demonstrated that the gradients in both the convolutional and fully-connected layers of a neural network contain ample information about the input data and the associated ground-truth labels to enable their reconstruction. The purpose of the following analysis is to establish that the gradient in the \ac{BN} layer also encodes information about the input data and the ground-truth label.

\begin{definition}
\textbf{\ac{BN} Layer:} To address the vanishing gradient problem arising from input values reaching the saturation zone of the non-linear function, the \ac{BN} layer is utilized for each hidden neuron in a \ac{DNN}. This layer normalizes the input values to a standard normal distribution, which shifts the input values to a region that is more responsive to the input and helps mitigate the vanishing gradient problem.
\end{definition}
Suppose there is a mini-batch of input data: $\mathcal{B}= \{u_1...u_N\}$, where $N$ is the batch size. The \ac{BN} function can be defined as:
\begin{align}
    \mu_{\mathcal{B}} &= \frac{1}{N} \sum_n u_n \\
    \sigma^2_{\mathcal{B}} &= \frac{1}{N} \sum_n (u_n - \mu_{\mathcal{B}})^2 \\
    \hat{u} &= \frac{u - \mu_{\mathcal{B}}}{\sqrt{\sigma^2_{\mathcal{B}}+\epsilon}} \\
    s &= \gamma \cdot \hat{u} + \beta \label{equ_s}
\end{align}
$epsilon$ is a small value to avoid dividing by zero.

\begin{claim}
\label{claim:4}
    The gradient in the \ac{BN} layer contains information about the input data and ground-truth labels.
\end{claim}
\begin{proof}
    We modify the network in Claim~\ref{claim:3} by adding a \ac{BN} layer after the convolutional layer:
    \begin{align}
        x &= input \nonumber\\
        u &= \theta \otimes x + b_1 \nonumber\\
        s &= BN_{\gamma,\beta}(u) \nonumber\\
        a &= ReLU(s) \nonumber\\
        a^{\prime} &= flatten(a) \nonumber\\
        z &= \lambda \cdot a^{\prime} + b_2 \nonumber\\
        \hat{y} &= softmax(z) \nonumber\\
        \mathcal{L} &= CE(y, \hat{y}) \nonumber
    \end{align}
    where $\gamma$ and $\beta$ are parameters in \ac{BN} layer. The derivative of the loss function with respect to $\gamma$ and $\beta$ are:
    \begin{align}
        \frac{\partial \mathcal{L}}{\partial \gamma} &= \frac{\partial \mathcal{L}}{\partial s} \frac{\partial s}{\partial \gamma} 
        = (\hat{y}-y) \cdot \lambda \circ sgn(a) \cdot \hat{u} \label{equ:gradient_loss_gamma} %\nonumber\\
        %\nonumber 
        \\
        \frac{\partial \mathcal{L}}{\partial \beta} &= \frac{\partial \mathcal{L}}{\partial s} \frac{\partial s}{\partial \beta} 
        = (\hat{y}-y) \cdot \lambda \circ sgn(a) \label{equ:gradient_loss_beta}%\nonumber
    \end{align}    
    Same as Claims~\ref{claim:1}, \ref{claim:2} and \ref{claim:3}, we can show that the gradient of \ac{BN} layer is correlated to the input data and ground-truth label. However, to compute the gradient of the network accurately using the \emph{Chain Rule}, it is necessary to calculate the derivative of the loss function with respect to the input of the \ac{BN} layer, $\frac{\partial \mathcal{L}}{\partial u}$. This computation is not trivial, as we know that the normalized output $\hat{u}$, variance value $\sigma^2_{\mathcal{B}}$ and mean value $\mu_{\mathcal{B}}$ are all related to the input $u$. So we first compute the derivative of $\frac{\partial \mathcal{L}}{\partial \hat{u}}$, $\frac{\partial \mathcal{L}}{\partial \sigma^2_{\mathcal{B}}}$ and $\frac{\partial \mathcal{L}}{\partial \mu_{\mathcal{B}}}$:
    \begin{align}
        \frac{\partial \mathcal{L}}{\partial \hat{u}} &= \frac{\partial \mathcal{L}}{\partial s} \frac{\partial s}{\partial \hat{u}} 
        = (\hat{y}-y) \cdot \lambda \circ sgn(a) \cdot \gamma \label{equ:gradient_loss_hat_u} \\%\nonumber \\ 
        \nonumber \\
        \frac{\partial \mathcal{L}}{\partial \sigma^2_{\mathcal{B}}} &= \frac{\partial \mathcal{L}}{\partial s} \frac{\partial s}{\partial \hat{u}} \frac{\partial \hat{u}}{\partial \sigma^2_{\mathcal{B}}} \label{equ:gradient_loss_sigma}\\
            &= -\frac{1}{2} \frac{\partial \mathcal{L}}{\partial \hat{u}} \cdot (u - \mu_{\mathcal{B}})\cdot (\sigma^2_{\mathcal{B}} + \epsilon)^{-\frac{3}{2}} \nonumber \\
        \nonumber \\
        \frac{\partial \mathcal{L}}{\partial \mu_{\mathcal{B}}} &=  \frac{\partial \mathcal{L}}{\partial \hat{u}} \frac{\partial \hat{u}}{\partial \mu_{\mathcal{B}}} + \frac{\partial \mathcal{L}}{\partial \sigma^2_{\mathcal{B}}} \frac{\partial \sigma^2_{\mathcal{B}}}{\partial \mu_{\mathcal{B}}} \label{equ:gradient_loss_mu}\\
            &= \frac{\partial \mathcal{L}}{\partial \hat{u}} \cdot \frac{-1}{\sqrt{\sigma^2_{\mathcal{B}} + \epsilon}} + \frac{\partial \mathcal{L}}{\partial \sigma^2_{\mathcal{B}}} \cdot \frac{-2(u - \mu_{\mathcal{B}})}{N} \nonumber 
    \end{align}
    Then we have:
    \begin{align}
        \frac{\partial \mathcal{L}}{\partial u} &= \frac{\partial \mathcal{L}}{\partial \hat{u}} \frac{\partial \hat{u}}{\partial u} + \frac{\partial \mathcal{L}}{\partial \sigma^2_{\mathcal{B}}} \frac{\partial \sigma^2_{\mathcal{B}}}{\partial u} + \frac{\partial \mathcal{L}}{\partial \mu_{\mathcal{B}}} \frac{\partial \mu_{\mathcal{B}}}{\partial u} \label{equ:gradient_loss_un}\\
            &= \frac{\partial \mathcal{L}}{\partial \hat{u}} \cdot \frac{1}{\sqrt{\sigma^2_{\mathcal{B}} + \epsilon}} + \frac{\partial \mathcal{L}}{\partial \sigma^2_{\mathcal{B}}} \cdot \frac{2(u - \mu_{\mathcal{B}})}{N} + \frac{\partial \mathcal{L}}{\partial \mu_{\mathcal{B}}} \cdot \frac{1}{N} \nonumber 
    \end{align}
    In a typical \ac{FL} setting, a malicious server has access to the gradients of each layer, allowing for the calculation of all variables in the equations presented above. The gradients obtained from Eqns.~\ref{equ:gradient_loss_theta_prime} \& \ref{equ:gradient_loss_gamma} contain rich information about the input data in both shallow and deep layers. Additionally, Eqn.~\ref{equ:gradient_loss_gamma} reveals that the trainable parameter $\gamma$ is highly related to the input feature map $\hat{u}$, as observed in Eqns.~\ref{equ:gradient_loss_sigma} \& \ref{equ:gradient_loss_mu}. Consequently, the derivative of the loss function with respect to the output of the convolutional layer $u$, as shown in Eqn.~\ref{equ:gradient_loss_un}, contains sufficiently rich input data information to enable deep leakage attacks.
\end{proof}

Referring to all the reasoning processes above, we can conclude that for any arbitrary \ac{CNN}: 
\begin{proposition}
    The gradients of the fully-connected layer, convolutional layer, and \ac{BN} layer in an image classification task contain sufficient information about the input data and ground-truth label, making it possible for an attacker to reconstruct them by regressing the gradients.
\end{proposition}

%----------------------------------------------------------------------
\subsection{Theoretical Analysis on Proposed Leakage Defense}
\begin{proposition}
\label{proposition:2}
    For a \ac{CNN}, embedding a key-lock module into the model can prevent the inheritance of the private input information through the gradient, where this is achieved by dropping the information of the input data in the gradient.
\end{proposition}
\begin{proof}
    Theoretical analysis is conducted to investigate how the key-lock module prevents the inheritance of private information throughout the gradient in a \ac{CNN}. Unlike in the traditional \ac{BN} layer where $\lambda$ and $\beta$ are trainable parameters, the key-lock module generates them as the output of two fully-connected layers called lock, which take a private input sequence called key as input. The proposed approach introduces new trainable parameters, namely $\omega$, $\phi$, and two biases $b_2$ and $b_3$. The network is defined as follows:
    \begin{align}
        x &= input \nonumber\\
        k &= key \nonumber\\
        u &= \theta \otimes x + b_1 \nonumber\\
        \gamma &= \omega \cdot k + b_2 \nonumber\\
        \beta &= \phi \cdot k + b_3 \nonumber\\
        s &= BN_{\gamma,\beta}(u) \nonumber\\
        a &= ReLU(s) \nonumber\\
        a^{\prime} &= flatten(a) \nonumber\\
        z &= \lambda \cdot a^{\prime} + b_4 \nonumber\\
        \hat{y} &= softmax(z) \nonumber\\
        \mathcal{L} &= CE(y, \hat{y}) \nonumber
    \end{align}
    The derivative of the loss function with respect to $\omega$ and $\phi$ are:
    \begin{align}
        \frac{\partial \mathcal{L}}{\partial \omega} &= \frac{\partial \mathcal{L}}{\partial s} \frac{\partial s}{\partial \omega} 
        = (\hat{y}-y) \cdot \lambda \circ sgn(a) \cdot k \cdot \hat{u} \label{equ:26}%\nonumber\\
        \\
        \frac{\partial \mathcal{L}}{\partial \phi} &= \frac{\partial \mathcal{L}}{\partial s} \frac{\partial s}{\partial \phi} 
        = (\hat{y}-y) \cdot \lambda \circ sgn(a) \cdot k  \label{equ:27}%\nonumber
    \end{align}
    The gradient in the lock layer is found to be closely related to the key sequence and feature map, which makes it possible for an attacker to infer the gradient of the lock layer given the key and leak the input data. To address this issue, we propose to keep both the key sequence and the weight of the lock layer private in our proposed method, FedKL. Once embedding the key-lock module, then Eqn.~\ref{equ:gradient_loss_hat_u} can be rewritten:
    \begin{align}
        \label{equ:gradient_loss_hat_uu}
        \frac{\partial \mathcal{L}}{\partial \hat{u}} &= \frac{\partial \mathcal{L}}{\partial s} \frac{\partial s}{\partial \hat{u}} 
        = \frac{\partial \mathcal{L}}{\partial s} \cdot (\omega \cdot k + b_2)  \\
        &= (\hat{y}-y) \cdot \lambda \circ sgn(a) \cdot (\omega \cdot k + b_2)\nonumber
    \end{align}
    The key sequence $k$, lock layer's weight $\omega$, and bias $b_2$ are considered confidential in our proposed FedKL. $\phi$ and $b_3$ are omitted in back-propagation. According to Eqn.~\ref{equ_s}, the derivative of $s$ with respect to $\hat{u}$ equals $\gamma$. Therefore, in Eqn.~\ref{equ:gradient_loss_hat_uu}, $\phi$ and $b_3$ are omitted due to the omission of $\beta$. In contrast to Section~\ref{sec:TAGL} Claim~\ref{claim:4}, the unknown $\gamma$ and $\beta$ make it impossible to compute the components in Eqns. \ref{equ:gradient_loss_sigma}, \ref{equ:gradient_loss_mu} \& \ref{equ:gradient_loss_un}. Similarly, in Eqn.~\ref{equ:gradient_loss_gamma}, $\hat{u}$ can be computed, but $s$ cannot be determined in Eqn.~\ref{equ_s}. Therefore, it is not possible to infer the input information from the feature map $s$. The proposed key-lock module effectively prevents the feature map from inheriting the input information and carrying it throughout the gradient when the private key and parameter of the lock module are unknown to the attacker.
\end{proof}

%----------------------------------------------------------------------
\subsection{Key-Lock Module}

A \ac{BN} layer typically succeeds a convolutional layer and encompasses two processing stages: normalization followed by scaling and shifting. The coefficients for scale and shift comprise two trainable parameters that are updated via back-propagation. Stemming from the theoretical examination earlier, we propose a key-lock module to defend against gradient leakage attacks in \ac{FL}, referred to as FedKL. This module can be incorporated subsequent to the convolution-normalization block, wherein the scale factor ($\gamma \in \mathbb{R}^O$) and the shift bias ($\beta \in \mathbb{R}^O$) cease to be trainable parameters and instead become outputs of the key-lock module. Here, $O$ denotes the number of output channels of the convolutional layer. The input of the convolutional layer is represented as $X_{conv} \in \mathbb{R}^{B\times C\times W\times H}$, with $B$ denoting the batch size, $C$ representing the number of channels, and $W$ and $H$ signifying the width and height of the input image, respectively. The weight in the convolutional layer is denoted as $W_{conv} \in \mathbb{R}^{C\times O\times K\times K}$, with $K$ indicating the kernel size. We define the input key, $X_{key} \in \mathbb{R}^S$, as having a length of $S$, and the weight of the key-lock module as $W_{lock} \in \mathbb{R}^{S \times O}$. The bias is represented as $b$. Therefore, the embedding transformation $\mathcal{F}()$ can be formulated as:
\begin{align}
\mathcal{F}(X_{conv}, X_{key}) &= \gamma \cdot (X_{conv} \otimes W_{conv}  ) + \beta \nonumber\\
\gamma &= X_{key} \cdot W_{lock-\gamma} + b_{lock-\gamma} \nonumber\\
\beta &= X_{key} \cdot W_{lock-\beta} +b_{lock-\beta} %\nonumber
\end{align}
where the convolution operation is denoted by $\otimes$, while the inner product is represented by $\cdot$. It is important to note that $\gamma$ and $\beta$ in the general normalization layer are trainable parameters; however, in our proposed methodology, they are the outputs of the lock layer, given the private key sequence. Fig.~\ref{fig:key-lock} demonstrates a typical \ac{DNN} building block, incorporating both a general convolution-normalization module and the proposed key-lock module.

\begin{figure}[t!]
    \centering
    \includegraphics[width=0.99\linewidth]{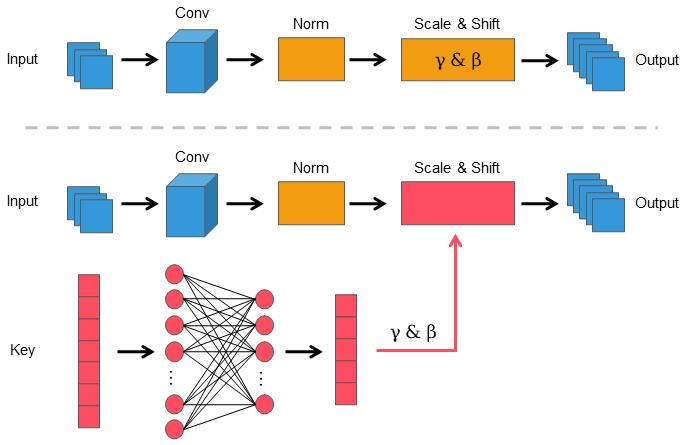}
    \caption{The top is the illustration of general convolution-normalization block, while the bottom indicates a key-lock module embedded in the block for generating the parameters, $\gamma$ and $\beta$.}
    \label{fig:key-lock}
    % \vspace{-8pt}
\end{figure}

%----------------------------------------------------------------------
\subsection{FL with Key-Lock Module}

\begin{figure}[!ht]
    \centering
    \includegraphics[width=0.99\linewidth]{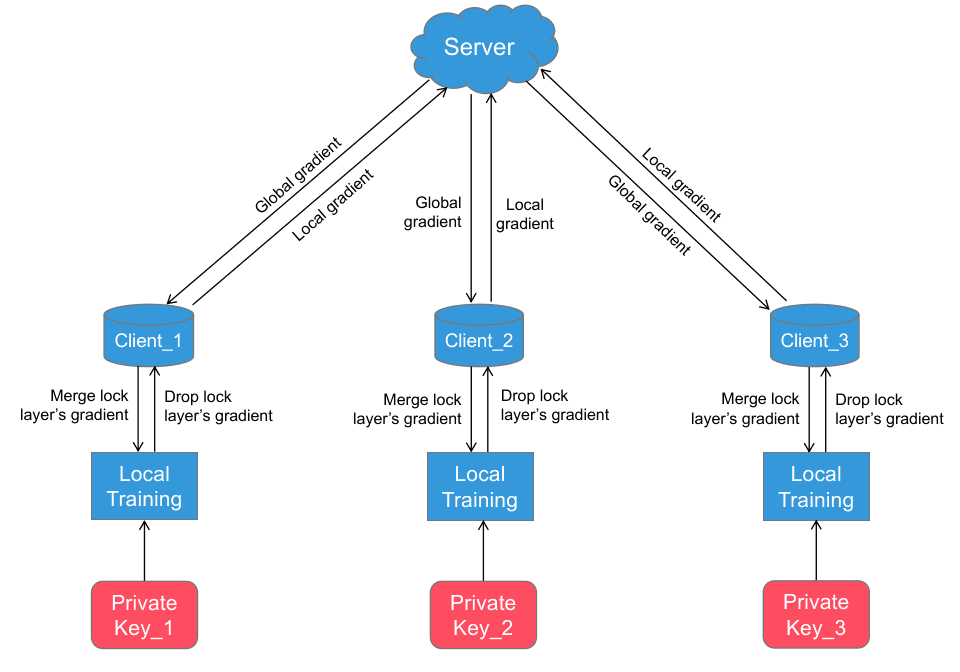}
    \caption{FL system with the key-lock module.}
    \label{fig:FLKL}
    % \vspace{-8pt}
\end{figure}

It is crucial to emphasize that the proposed key-lock module solely modifies the \ac{DNN} architecture and does not require alterations to the \ac{FL} strategy. In a conventional \ac{FL} system, the global model is initially instantiated on the server and subsequently distributed to each client for secure local training. Local gradients from clients are aggregated on the server, converging into a new global model for the subsequent round of local training. To implement the proposed key-lock module within the system, the server first initializes the model with the key-lock module and then distributes the model to each client. Individual clients generate their own private key sequences for local training. At each periodic aggregation stage, the gradient of the model, excluding the lock layer, is merged at the server to generate the new global model. Throughout the process, both the key sequence and lock layer of each client are maintained as strictly confidential and securely localized. The global model does not retain any information from a client's key-lock module; therefore, the global model and gradient from the local client cannot be employed for data reconstruction attacks. In the subsequent round of local training, each client first updates its own model using the distributed global model, without updating the lock layer. The client's lock layer is then integrated into the new model. Fig.~\ref{fig:FLKL} details a \ac{FL} system incorporating the key-lock module.

%=====================================================================
\section{Experiments}
\label{sec:er}

In this section, we first introduce the datasets and metrics employed for benchmark evaluation. Then, we present three sets of experiments to assess: 1) the influence of prediction accuracy of \ac{DNN}s with the proposed key-lock module; 2) the effectiveness of defending against state-of-the-art gradient leakage attacks, including \ac{DLG}, \ac{GRNN}, and \ac{GGL}; 3) the ablation study of various key hyperparameters of \ac{DNN}s and training settings.

%-----------------------------------------------------------------------
\begin{table*}[!ht]
\setlength{\tabcolsep}{8pt}
\begin{center}
\caption{Testing accuracy(\%) from models trained at different settings. The red ones are the highest accuracy and the blue ones are the next highest. C-10 represents CIFAR-10 and C-100 is CIFAR-100. KL means key-lock module. The accuracy results in the Reference row are taken from the relevant papers.}
\label{tab:accuracy}
\begin{tabular}{c c|c|c c|c c|c c|c c|c c}
\hline
\multicolumn{2}{c|}{\textbf{Model}} & \makecell[c]{\emph{LeNet}\\(32*32)} & \multicolumn{2}{c|}{\makecell[c]{\emph{ResNet-20}\\(32*32)}} & \multicolumn{2}{c|}{\makecell[c]{\emph{ResNet-32}\\(32*32)}} & \multicolumn{2}{c|}{\makecell[c]{\emph{ResNet-18}\\(224*224)}} & \multicolumn{2}{c|}{\makecell[c]{\emph{ResNet-34}\\(224*224)}} & \multicolumn{2}{c}{\makecell[c]{\emph{VGG-16}\\(224*224)}} \\
\hline
\multicolumn{2}{c|}{\textbf{Dataset}} & MNIST & C-10 & C-100 & C-10 & C-100 & C-10 & C-100 & C-10 & C-100 & C-10 & C-100 \\
\hline
\multirow{2}{*}{\textbf{Centralized}} & \textbf{w/o KL} & 98.09 & {\color{red}91.63} & {\color{red}67.59} & {\color{blue}92.34} & {\color{red}70.35} & 91.62 & 72.15 & 92.20 & 73.21& 89.13 & 63.23 \\
& \cellcolor{red!10} \textbf{w/ KL} & \cellcolor{red!10} 98.07 &\cellcolor{red!10}  90.58 & \cellcolor{red!10} {\color{blue}67.49} & \cellcolor{red!10} 91.05 & \cellcolor{red!10} {\color{blue}69.89} & \cellcolor{red!10} {\color{red}93.12} & \cellcolor{red!10} {\color{red}75.90} & \cellcolor{red!10} {\color{red}94.68} & \cellcolor{red!10} {\color{red}78.22} & \cellcolor{red!10} {\color{red}93.84} & \cellcolor{red!10} {\color{red}74.86}\\
% \cline{1-2}
% \hline
\multirow{2}{*}{\textbf{FL}} & \textbf{FedAvg} & {\color{blue}98.14} & 91.20 & 58.58 & 91.37 & 61.91 & 89.50 & 68.59 & 89.27 & 68.30 & 88.13 & 61.91 \\
& \cellcolor{blue!10} \textbf{FedKL} & \cellcolor{blue!10} 97.45 & \cellcolor{blue!10} 88.45 & \cellcolor{blue!10} 61.97 & \cellcolor{blue!10} 89.29 & \cellcolor{blue!10} 64.17 & \cellcolor{blue!10} {\color{blue}91.66} & \cellcolor{blue!10} {\color{blue}74.19} & \cellcolor{blue!10} {\color{blue}93.27} & \cellcolor{blue!10} {\color{blue}76.61} & \cellcolor{blue!10} {\color{blue}93.54} & \cellcolor{blue!10} {\color{blue}73.06} \\
\hline
\multicolumn{2}{c|}{\textbf{Reference}} & \makecell[c]{{\color{red}99.05}\\\cite{lecun1998gradient}} & \makecell[c]{{\color{blue}91.25}\\\cite{he2016deep}} & - & \makecell[c]{{\color{red}92.49}\\\cite{he2016deep}} & - & - & - & - & - & - & -\\
\hline
\end{tabular}
\end{center}
% \vspace{-8pt}
\end{table*}

\begin{table*}[!ht]
\setlength{\tabcolsep}{3pt}
\begin{center}
\caption{Testing accuracies and losses over epochs on CIFAR10 and CIFAR100 datasets using different networks. The models are all trained in centralized mode. The red solid lines are the results from models with the key-lock module and the blue dashed lines are those from models without the key-lock module.}
\label{tab:acc_loss}
\begin{tabular}{c c c c c c c}
&& \emph{ResNet-20} & \emph{ResNet-32} & \emph{ResNet-18} & \emph{ResNet-34} & \emph{VGG-16} \\
\multirow{7}{*}{\rotatebox{90}{\textbf{CIFAR10}}} & \makecell*[c]{\rotatebox{90}{Accuracy}} & \makecell*[c]{\includegraphics[width=0.17\linewidth]{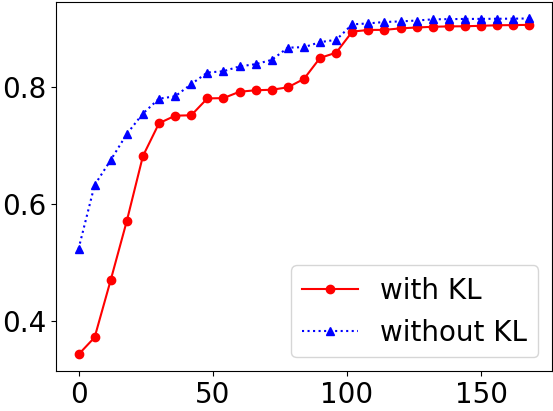}} & \makecell*[c]{\includegraphics[width=0.17\linewidth]{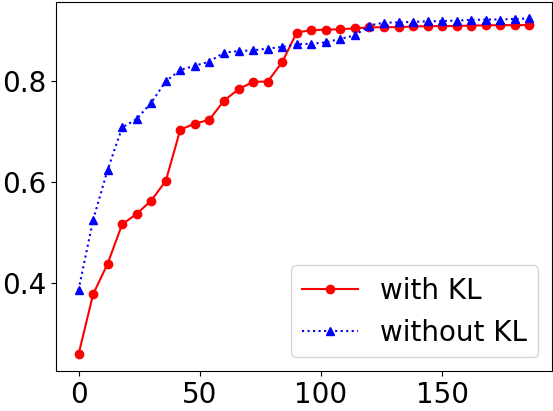}} & \makecell*[c]{\includegraphics[width=0.17\linewidth]{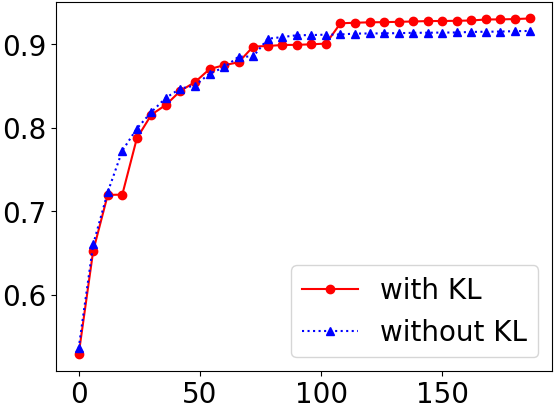}} & \makecell*[c]{\includegraphics[width=0.17\linewidth]{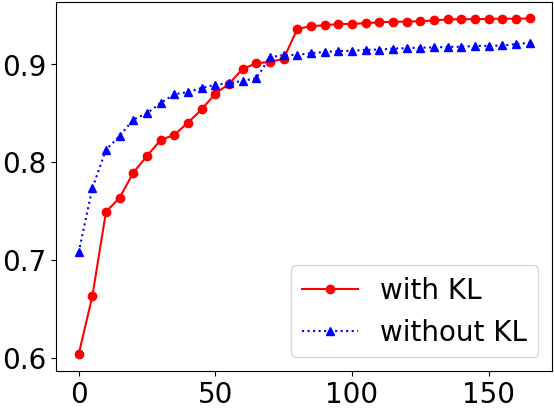}} & \makecell*[c]{\includegraphics[width=0.17\linewidth]{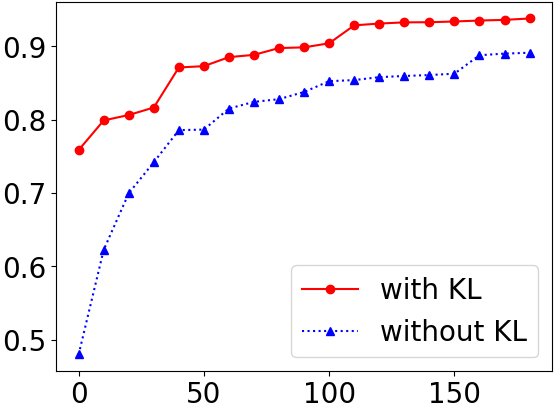}} \\
& \makecell*[c]{\rotatebox{90}{Loss}} & \makecell*[c]{\includegraphics[width=0.17\linewidth]{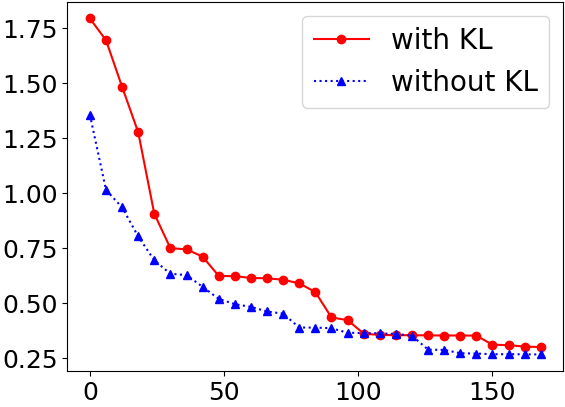}} & 
\makecell*[c]{\includegraphics[width=0.17\linewidth]{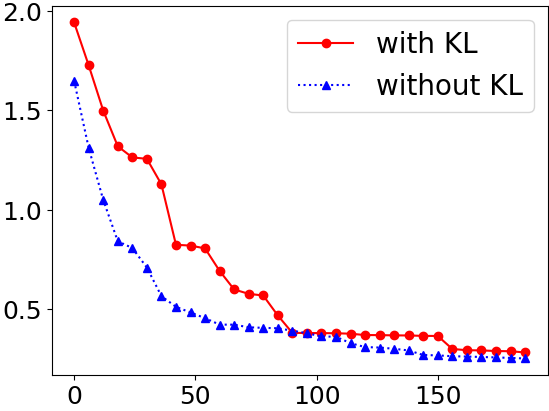}} & \makecell*[c]{\includegraphics[width=0.17\linewidth]{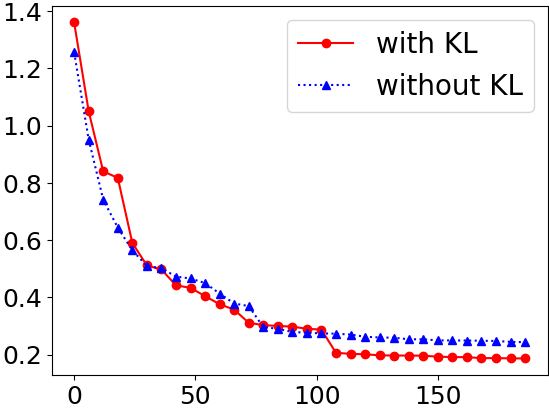}} & \makecell*[c]{\includegraphics[width=0.17\linewidth]{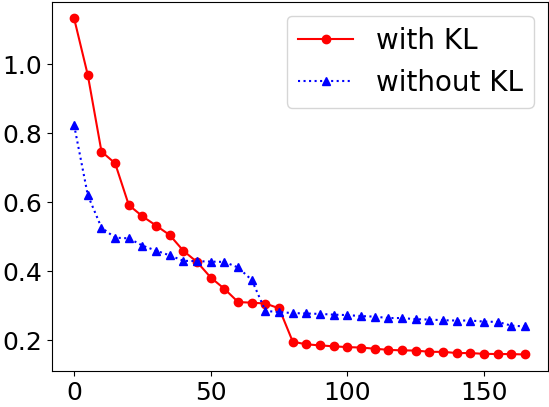}} & \makecell*[c]{\includegraphics[width=0.17\linewidth]{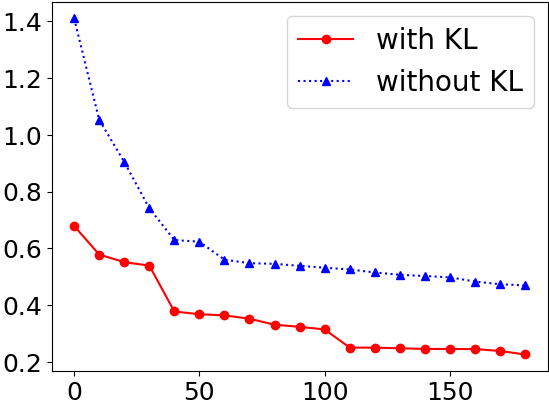}} \\

\multirow{6}{*}{\rotatebox{90}{\textbf{CIFAR100}}} & \makecell*[c]{\rotatebox{90}{Accuracy}} & \makecell*[c]{\includegraphics[width=0.17\linewidth]{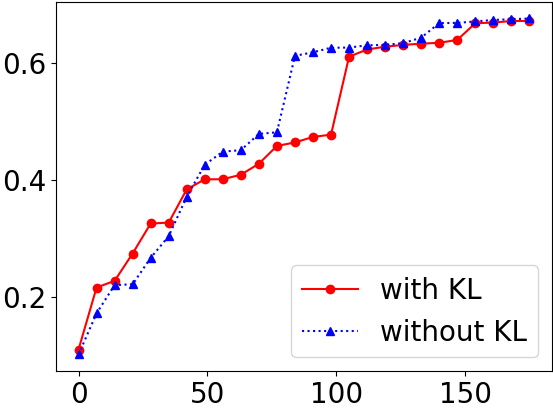}} & \makecell*[c]{\includegraphics[width=0.17\linewidth]{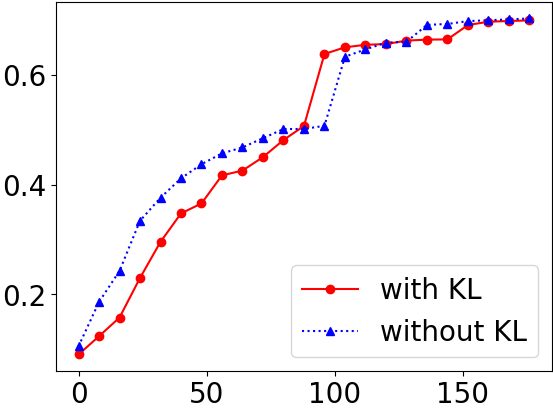}} & \makecell*[c]{\includegraphics[width=0.17\linewidth]{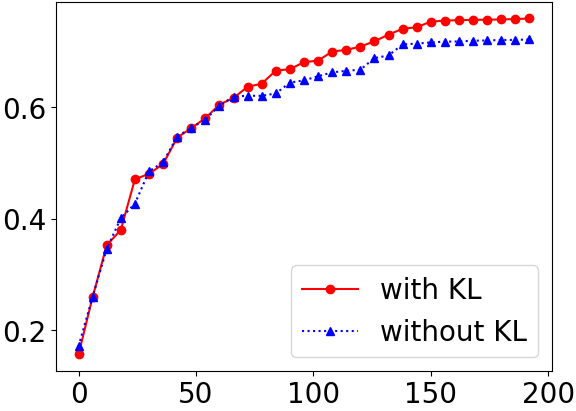}} & \makecell*[c]{\includegraphics[width=0.17\linewidth]{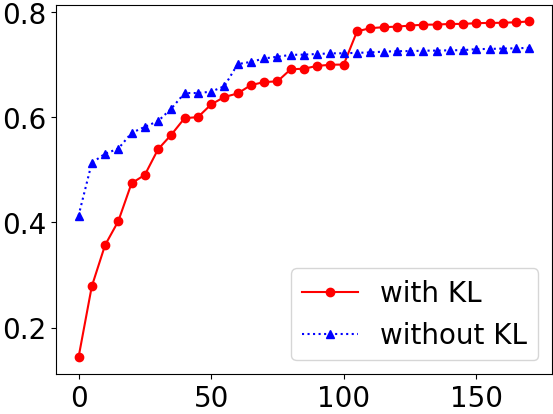}} & \makecell*[c]{\includegraphics[width=0.17\linewidth]{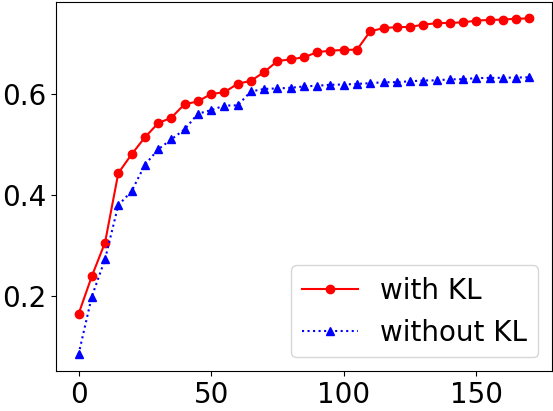}} \\
& \makecell*[c]{\rotatebox{90}{Loss}} & \makecell*[c]{\includegraphics[width=0.17\linewidth]{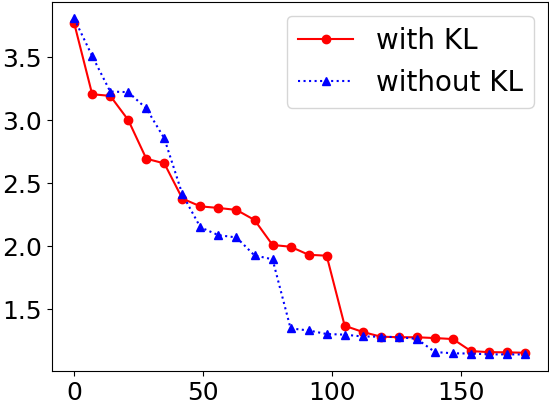}} & 
\makecell*[c]{\includegraphics[width=0.17\linewidth]{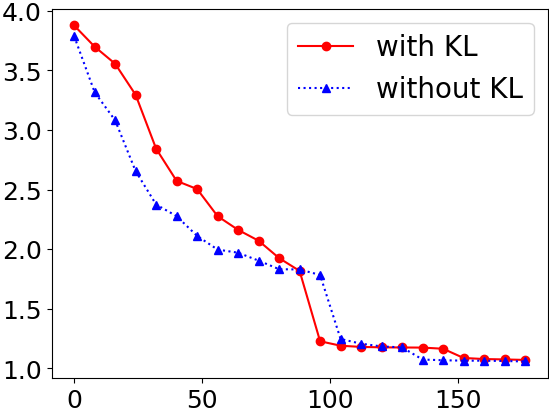}} & \makecell*[c]{\includegraphics[width=0.17\linewidth]{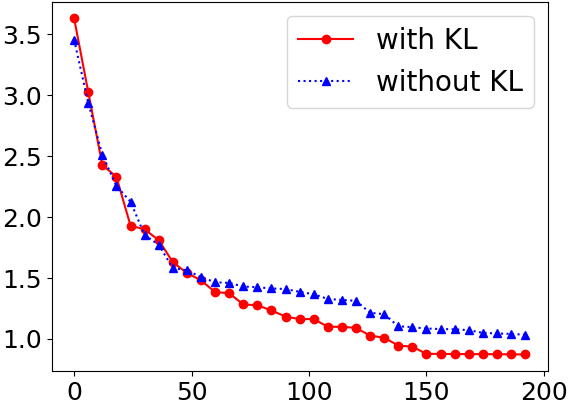}} & \makecell*[c]{\includegraphics[width=0.17\linewidth]{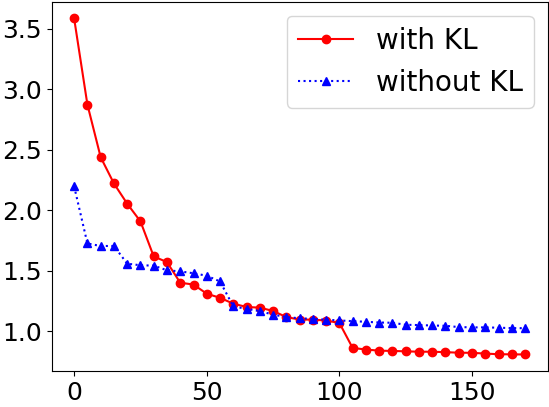}} & \makecell*[c]{\includegraphics[width=0.17\linewidth]{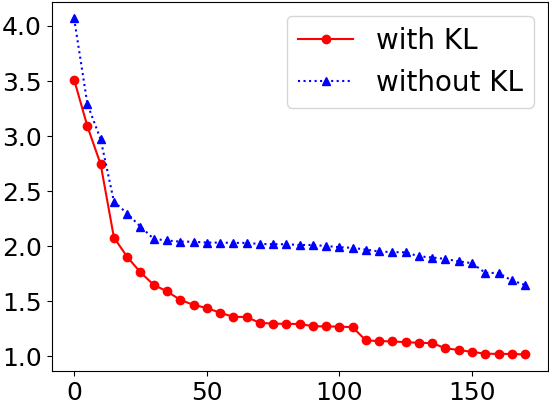}}
\end{tabular}
\end{center}
% \vspace{-8pt}
\end{table*}

\subsection{Benchmarks and Metrics}
Our experiments utilize four popular public benchmarks: MNIST~\cite{lecun1998mnist}, CIFAR-10~\cite{krizhevsky2009learning}, CIFAR-100~\cite{krizhevsky2009learning} and ILSVRC2012~\cite{deng2009imagenet}. Four evaluation metrics are employed to quantify the quality of the generated image, including \ac{MSE}, \ac{PSNR}, \ac{LPIPS}~\cite{zhang2018unreasonable} and \ac{SSIM}~\cite{wang2004image}. 

To allow maximum capacity for gradient leakage attacks and hence to evaluate defense capability, the batch size is set to 1. Consequently, the reconstructed image and the true image can be directly matched. MSE measures the pixel-wise L2 difference between the reconstructed image and the true image. Formally, \ac{MSE} is expressed as $MSE(X, \hat{X}) = ||X - \hat{X}||_2$, where $X$ denotes the true image and $\hat{X}$ represents the reconstructed image. \ac{PSNR} is a criterion for image quality assessment, defined as the logarithm of the ratio of the squared maximum value of RGB image fluctuation over MSE between two images. The definition is given as: $PSNR(X, \hat{X}) = 10 \cdot \lg(\frac{255^2}{MSE(X, \hat{X})})$. A higher \ac{PSNR} score indicates greater similarity between the two images. \ac{LPIPS} essentially computes the similarity between two image patches using predefined networks. We employ two neural networks, \emph{VGGNet}~\cite{simonyan2014very} and \emph{AlexNet}~\cite{krizhevsky2012imagenet}, to perceptually evaluate the similarity of two images. This approach is validated to be consistent with human perception. A lower \ac{LPIPS} score signifies that the two images are perceptually more similar. \ac{SSIM} assesses the proximity of image distortion by detecting changes in structural information. A higher \ac{SSIM} score implies a better match. The proposed key-lock module is exclusively embedded in the first convolution-normalization block for all experiments in this study. 

The length of the randomly generated key sequence is $1024$. The output dimension is $16$ for \emph{ResNet-20} and \emph{ResNet-32} on CIFAR-10 and CIFAR-100 datasets, and $64$ for \emph{ResNet-18} and \emph{ResNet-34} on ILSVRC2012. The framework used to implement the neural network models is PyTorch~\cite{paszke2019pytorch}. The optimizer is \ac{SGD} with an initial learning rate of $0.01$, a decay of $5 \times 10^{-4}$, and a momentum of $0.9$. The learning rate is reduced by a factor of $0.2$ at epochs (rounds) $60$, $120$, and $160$. The total number of training epochs (rounds) is $200$. Our implementation of proposed FedKL is publicly available \footnote{\url{https://github.com/Rand2AI/FedKL}}. 

%------------------------------------------------------------------------

\subsection{Accuracy Impact}
In order to assess the impact on performance, we devised two approaches: 1) centralized training with or without the key-lock module, and 2) collaborative training within the \ac{FedAvg} framework, also with or without the key-lock module. We employed \emph{LeNet}~\cite{lecun1998gradient}, \emph{VGGNet}~\cite{simonyan2014very}, and \emph{ResNet}~\cite{he2016deep} as the fundamental networks for training image classification models. A total of forty-four models were trained, utilizing eleven experimental configurations and two distinct training methodologies (centralized and federated), to compare standard networks and those incorporating key-lock modules. The findings are presented in Table~\ref{tab:accuracy}.

Overall, the most favorable results are observed when centralized training is implemented. For instance, \emph{LeNet} on MNIST yields the highest performance of $99.05\%$ accuracy according to paper~\cite{lecun1998gradient}, while \emph{ResNet-20} achieves the maximum accuracy of $91.63\%$ on CIFAR-10 and $67.59\%$ on CIFAR-100. The addition of the key-lock module does not lead to a significant decline in model inference performance. In certain instances, models incorporating the key-lock module even outperform those without it. For example, \emph{ResNet-20} on CIFAR-100 attains a testing accuracy of $61.97\%$ within the FedKL framework, which is $5.79\%$ higher than the conventional \ac{FedAvg} model that achieves $58.58\%$. The most substantial performance enhancement is observed with \emph{VGG-16} on CIFAR-100, where the FedKL model attains an accuracy of $73.06\%$, an improvement of $18.43\%$ over the $61.69\%$ achieved using \ac{FedAvg}. Conversely, for experiments utilizing images with a resolution of $32*32$, optimal performance is obtained with centralized configurations without the key-lock module. Nevertheless, we note that the centralized setup incorporating the proposed key-lock module achieves the highest prediction accuracy when the resolution is increased to $224*224$. Moreover, the outcomes from \emph{ResNet-18}, \emph{ResNet-34}, and \emph{VGG-16} indicate that models with the key-lock module significantly outperform those without it in both centralized and \ac{FL} systems.

In Table~\ref{tab:acc_loss}, we illustrate testing accuracies and losses in relation to training epochs for various networks using the CIFAR-10 and CIFAR-100 datasets. All models are trained in the centralized mode. The performance disparities between models with and without the key-lock module are relatively minor. In particular, the convergence trends exhibit a strong correlation in numerous instances, such as with \emph{ResNet-20} and \emph{ResNet-32} trained on CIFAR-100, and \emph{ResNet-18} trained on CIFAR-10. As mentioned previously, when higher resolution is employed, improved prediction performance can be observed utilizing the proposed key-lock module. It is logical to infer that this performance enhancement stems from the increased learning capacity derived from the additional two fully-connected layers present in the key-lock module.

To further understand the impact of the key-lock module on model performance, a comparison was conducted between local and global models by inputting a randomly generated key sequence and employing the initial weights of the lock layer on the server at each training round. Given that the proposed FedKL learns a personalized model due to the unique private key, the accuracies of FedKL, as presented in Table~\ref{tab:accuracy}, are averaged over all local clients. Each client's accuracy is derived from their privately owned key sequence and trained key-lock layer's parameters. The optimal model is selected from all training rounds, with the average accuracy representing the mean of all clients' optimal performances. The ``Random'' row results are obtained by utilizing a randomly generated key sequence and random lock layer weights.

\begin{table}[ht!]
\setlength{\tabcolsep}{3pt}
\begin{center}
\caption{Accuracy (\%) of models trained by FedKL with different input key sequences. The parameter proportion means the ratio of the parameter number in lock layer over the total parameter number in the model.}
\label{tab:FedKL_accuracy}

\begin{tabular}{c | c | c c | c}
\hline
\multirow{2}{*}{\textbf{Model}}  & \multirow{2}{*}{\textbf{Key Source}} & \multicolumn{2}{c|}{\textbf{MNIST}} & \multirow{2}{*}{\makecell[c]{\textbf{Parameter}\\\textbf{Proportion}}} \\
\cline{3-4}
&&\textbf{Acc.} & \textbf{Gap} &\\
\hline
\multirow{5}{*}{\makecell[c]{\emph{LeNet}\\(32*32)}} 
 & Client 0 & 97.88 & 87.43 
 & \multirow{5}{*}{22.43\%} \\
 & Client 1 & 97.90 & 87.45 \\
 & Client 2 & 97.66 & 87.21 \\
 & \cellcolor{red!10} Average& \cellcolor{red!10} 97.45 & \cellcolor{red!10} 87.00 \\
 & \cellcolor{blue!10} Random& \cellcolor{blue!10} 10.45 & \cellcolor{blue!10} 0 \\
\hline
\end{tabular}

\begin{tabular}{c}
~\\
\end{tabular}

\begin{tabular}{c | c | c c | c c | c}
\hline
\multirow{2}{*}{\textbf{Model}}  & \multirow{2}{*}{\textbf{Key Source}} & \multicolumn{2}{c|}{\textbf{C-10}} & \multicolumn{2}{c|}{\textbf{C-100}} & \multirow{2}{*}{\makecell[c]{\textbf{Parameter}\\\textbf{Proportion}}} \\
\cline{3-6}
&& \textbf{Acc.} & \textbf{Gap} & \textbf{Acc.} & \textbf{Gap}\\
\hline
\multirow{5}{*}{\makecell[c]{\emph{ResNet-20}\\(32*32)}} 
 & Client 0 & 88.59 & 9.42 & 62.17 & 33.66 
 & \multirow{5}{*}{10.49\%}\\
 & Client 1  & 87.88 & 8.71 & 61.84 & 33.33 \\
 & Client 2  & 88.63 & 9.46 & 61.99 & 33.48 \\
 & \cellcolor{red!10} Average & \cellcolor{red!10} 88.45 & \cellcolor{red!10} 9.28 & \cellcolor{red!10} 61.97 & \cellcolor{red!10} 33.46 \\
 & \cellcolor{blue!10} Random & \cellcolor{blue!10} 79.17 & \cellcolor{blue!10} 0 & \cellcolor{blue!10} 28.51 & \cellcolor{blue!10} 0 \\
\hline
\multirow{5}{*}{\makecell[c]{\emph{ResNet-32}\\(32*32)}}
& Client 0  & 89.26 & 13.35 & 63.92 & 18.78
& \multirow{5}{*}{6.46\%}\\
& Client 1  & 89.42 & 13.51 & 64.28 & 19.14\\
& Client 2  & 88.83 & 12.92 & 64.17 & 19.03\\
& \cellcolor{red!10} Average  & \cellcolor{red!10} 89.29 & \cellcolor{red!10} 13.38 & \cellcolor{red!10} 64.17 & \cellcolor{red!10} 19.03\\
& \cellcolor{blue!10} Random  & \cellcolor{blue!10} 75.91 & \cellcolor{blue!10} 0 & \cellcolor{blue!10} 45.14 & \cellcolor{blue!10} 0\\
\hline
\multirow{5}{*}{\makecell[c]{\emph{ResNet-18}\\(224*224)}}
& Client 0 & 90.23 & 0.82 & 74.30 & 12.70
& \multirow{5}{*}{1.15\%}\\
& Client 1 & 91.78 & 2.37 & 74.17 & 12.57\\
& Client 2 & 89.63 & 0.22 & 74.28 & 12.68\\
& \cellcolor{red!10} Average & \cellcolor{red!10} 91.66 & \cellcolor{red!10} 2.25 & \cellcolor{red!10} 74.19 & \cellcolor{red!10} 12.59\\
& \cellcolor{blue!10} Random & \cellcolor{blue!10} 89.41 & \cellcolor{blue!10} 0 & \cellcolor{blue!10} 61.60 & \cellcolor{blue!10} 0\\
\hline
\multirow{5}{*}{\makecell[c]{\emph{ResNet-34}\\(224*224)}} 
& Client 0 & 93.40 & 1.63 & 76.74 & 6.98
& \multirow{5}{*}{0.61\%}\\
& Client 1 & 92.76 & 0.99 & 76.25 & 6.49\\
& Client 2 & 93.28 & 1.51 & 76.57 & 6.81\\
& \cellcolor{red!10} Average & \cellcolor{red!10} 93.27 & \cellcolor{red!10} 1.50 & \cellcolor{red!10} 76.61 & \cellcolor{red!10} 6.85\\
& \cellcolor{blue!10} Random & \cellcolor{blue!10} 91.77 & \cellcolor{blue!10} 0 & \cellcolor{blue!10} 69.76 & \cellcolor{blue!10} 0\\
\hline
\multirow{5}{*}{\makecell[c]{\emph{VGG-16}\\(224*224)}} 
& Client 0 & 93.58 & 0.36 & 72.67 & 1.07
& \multirow{5}{*}{0.097\%}\\
& Client 1 & 93.40 & 0.18 & 73.17 & 1.57\\
& Client 2 & 93.59 & 0.37 & 72.76 & 1.16\\
& \cellcolor{red!10} Average  & \cellcolor{red!10} 93.54 & \cellcolor{red!10} 0.32 & \cellcolor{red!10} 73.06 & \cellcolor{red!10} 1.46\\
& \cellcolor{blue!10} Random & \cellcolor{blue!10} 93.22 & \cellcolor{blue!10} 0 & \cellcolor{blue!10} 71.60 & \cellcolor{blue!10} 0\\
\hline
\end{tabular}
\end{center}
\end{table}

\begin{table*}[!ht]
\setlength{\tabcolsep}{3pt}
\begin{center}
\caption{Comparison of image reconstruction using \ac{DLG} and \ac{GRNN} with the key-lock module. In this case, the malicious server has no knowledge of the private key sequence and gradient of the lock layer. ``$\times$'' refers to a failure image reconstruction.}
\label{tab:dp}
\begin{tabular}{c | c | c c c | c c c | c c c }
\hline
\multicolumn{2}{c|}{Model} & \multicolumn{3}{c|}{\makecell[c]{\emph{LeNet}\\(32*32)}} & \multicolumn{3}{c|}{\makecell[c]{\emph{ResNet-20}\\(32*32)}} & \multicolumn{3}{c}{\makecell[c]{\emph{ResNet-18}\\(256*256)}} \\ 
\hline
\multicolumn{2}{c|}{Dataset} & MNIST & C-10 & C-100 & MNIST & C-10 & C-100 & C-10 & C-100 & ILSVRC \\
\hline
\multirow{11}{*}{DLG} & True & \makecell*[c]{\includegraphics[width=0.07\linewidth]{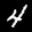}} & \makecell*[c]{\includegraphics[width=0.07\linewidth]{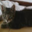}} & \makecell*[c]{\includegraphics[width=0.07\linewidth]{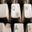}} & $\times$ & $\times$ & $\times$ & $\times$ & $\times$ & $\times$ \\
& w/o KL & \makecell*[c]{\includegraphics[width=0.07\linewidth]{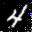}} & \makecell*[c]{\includegraphics[width=0.07\linewidth]{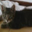}} & \makecell*[c]{\includegraphics[width=0.07\linewidth]{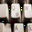}} & $\times$ & $\times$ & $\times$ & $\times$ & $\times$ & $\times$\\
& w/ KL & \makecell*[c]{\includegraphics[width=0.07\linewidth]{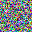}} & \makecell*[c]{\includegraphics[width=0.07\linewidth]{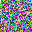}} & \makecell*[c]{\includegraphics[width=0.07\linewidth]{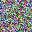}} & $\times$ & $\times$ & $\times$ & $\times$ & $\times$ & $\times$ \\
\hline
\multirow{11}{*}{GRNN} & True & \makecell*[c]{\includegraphics[width=0.07\linewidth]{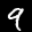}} & \makecell*[c]{\includegraphics[width=0.07\linewidth]{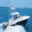}} & \makecell*[c]{\includegraphics[width=0.07\linewidth]{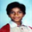}} & \makecell*[c]{\includegraphics[width=0.07\linewidth]{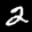}} & \makecell*[c]{\includegraphics[width=0.07\linewidth]{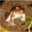}} & \makecell*[c]{\includegraphics[width=0.07\linewidth]{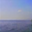}} & \makecell*[c]{\includegraphics[width=0.07\linewidth]{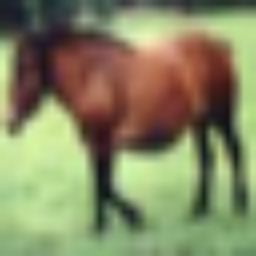}} & \makecell*[c]{\includegraphics[width=0.07\linewidth]{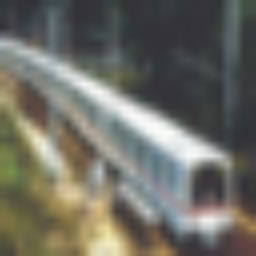}} & \makecell*[c]{\includegraphics[width=0.07\linewidth]{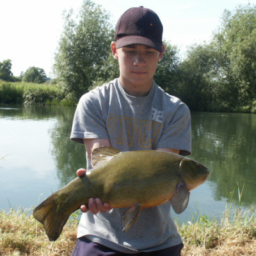}}\\
& w/o KL & \makecell*[c]{\includegraphics[width=0.07\linewidth]{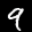}} & \makecell*[c]{\includegraphics[width=0.07\linewidth]{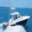}} & \makecell*[c]{\includegraphics[width=0.07\linewidth]{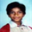}} & \makecell*[c]{\includegraphics[width=0.07\linewidth]{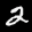}} & \makecell*[c]{\includegraphics[width=0.07\linewidth]{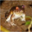}} & \makecell*[c]{\includegraphics[width=0.07\linewidth]{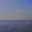}} & \makecell*[c]{\includegraphics[width=0.07\linewidth]{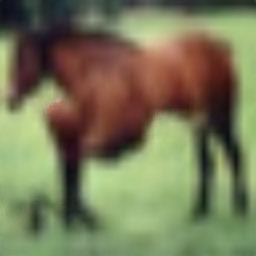}} & \makecell*[c]{\includegraphics[width=0.07\linewidth]{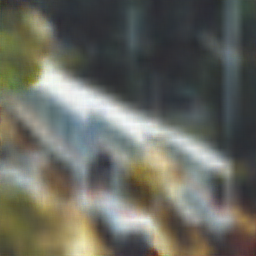}} & \makecell*[c]{\includegraphics[width=0.07\linewidth]{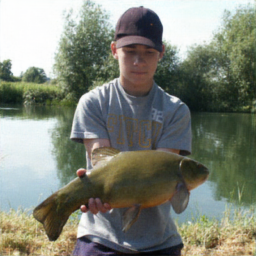}} \\
& w/ KL & \makecell*[c]{\includegraphics[width=0.07\linewidth]{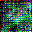}} & \makecell*[c]{\includegraphics[width=0.07\linewidth]{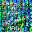}} & \makecell*[c]{\includegraphics[width=0.07\linewidth]{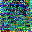}} & \makecell*[c]{\includegraphics[width=0.07\linewidth]{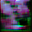}} & \makecell*[c]{\includegraphics[width=0.07\linewidth]{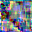}} & \makecell*[c]{\includegraphics[width=0.07\linewidth]{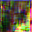}} & \makecell*[c]{\includegraphics[width=0.07\linewidth]{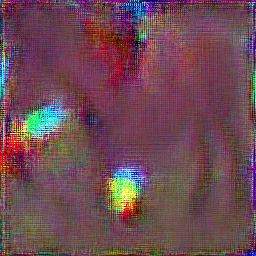}} & \makecell*[c]{\includegraphics[width=0.07\linewidth]{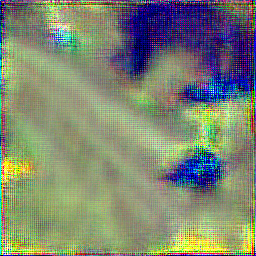}} & \makecell*[c]{\includegraphics[width=0.07\linewidth]{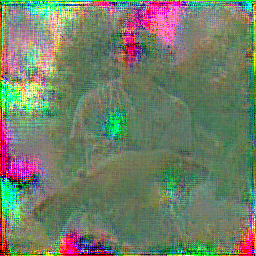}}\\
\hline
\multirow{11}{*}{IG} & True & \makecell*[c]{\includegraphics[width=0.07\linewidth]{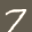}} & \makecell*[c]{\includegraphics[width=0.07\linewidth]{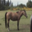}} & \makecell*[c]{\includegraphics[width=0.07\linewidth]{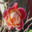}} & \makecell*[c]{\includegraphics[width=0.07\linewidth]{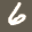}} & \makecell*[c]{\includegraphics[width=0.07\linewidth]{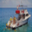}} & \makecell*[c]{\includegraphics[width=0.07\linewidth]{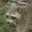}} & \makecell*[c]{\includegraphics[width=0.07\linewidth]{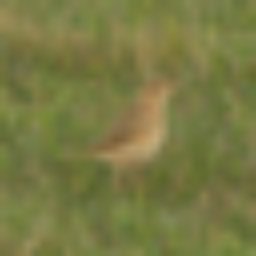}} & \makecell*[c]{\includegraphics[width=0.07\linewidth]{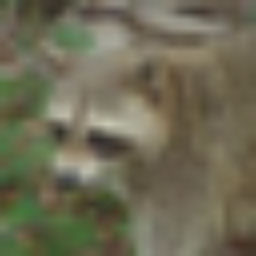}} & \makecell*[c]{\includegraphics[width=0.07\linewidth]{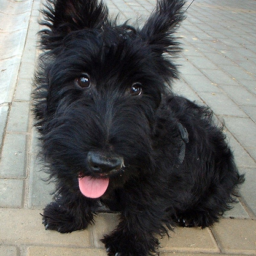}}\\
& w/o KL & \makecell*[c]{\includegraphics[width=0.07\linewidth]{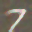}} & \makecell*[c]{\includegraphics[width=0.07\linewidth]{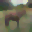}} & \makecell*[c]{\includegraphics[width=0.07\linewidth]{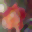}} & \makecell*[c]{\includegraphics[width=0.07\linewidth]{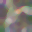}} & \makecell*[c]{\includegraphics[width=0.07\linewidth]{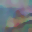}} & \makecell*[c]{\includegraphics[width=0.07\linewidth]{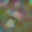}} & \makecell*[c]{\includegraphics[width=0.07\linewidth]{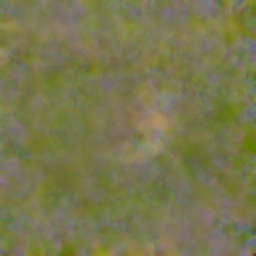}} & \makecell*[c]{\includegraphics[width=0.07\linewidth]{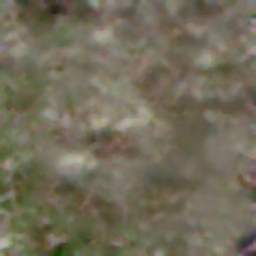}} & \makecell*[c]{\includegraphics[width=0.07\linewidth]{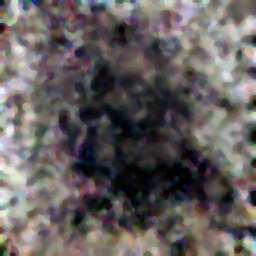}} \\
& w/ KL & \makecell*[c]{\includegraphics[width=0.07\linewidth]{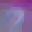}} & \makecell*[c]{\includegraphics[width=0.07\linewidth]{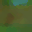}} & \makecell*[c]{\includegraphics[width=0.07\linewidth]{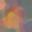}} & \makecell*[c]{\includegraphics[width=0.07\linewidth]{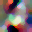}} & \makecell*[c]{\includegraphics[width=0.07\linewidth]{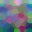}} & \makecell*[c]{\includegraphics[width=0.07\linewidth]{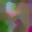}} & \makecell*[c]{\includegraphics[width=0.07\linewidth]{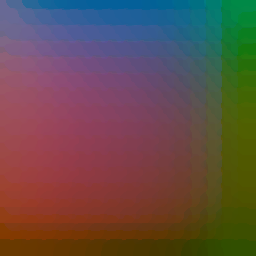}} & \makecell*[c]{\includegraphics[width=0.07\linewidth]{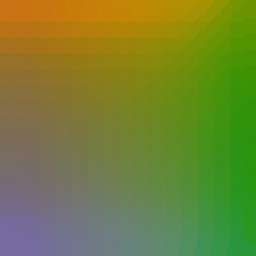}} & \makecell*[c]{\includegraphics[width=0.07\linewidth]{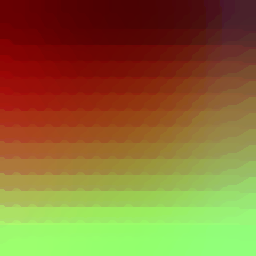}}\\
\hline
\end{tabular}
\end{center}
% \vspace{-8pt}
\end{table*}

In Table~\ref{tab:FedKL_accuracy}, the performance discrepancy is more pronounced for networks with fewer parameters compared to those with a higher number of parameters. The relationship between the lock layer's parameter count and the model's total parameter count was quantified, revealing a positive correlation between the performance gap and the weight proportion. As the lock layer parameter proportion decreases, the performance gap narrows. For instance, the smallest network, \emph{LeNet}, exhibits the highest parameter proportion of $22.43\%$ and the largest performance gap of $87.00\%$ on the model with the best average performance of $97.45\%$. The parameter proportion of \emph{ResNet-20} exceeds that of\emph{ResNet-32}, yet \emph{ResNet-20} outperforms \emph{ResNet-32} with a $33.46\%$ accuracy gap compared to $19.03\%$ on CIFAR-100. However, the performance gap of \emph{ResNet-20} is smaller than that of \emph{ResNet-32} ($9.28\%$ vs. $13.38\%$) on CIFAR-10. The accuracy gaps of \emph{ResNet-18} are larger than those of \emph{ResNet-34} ($2.25\%$ vs. $1.50\%$ on CIFAR-10 and $12.59\%$ vs. $6.85\%$ on CIFAR-100). \emph{VGG-16}, with the lowest parameter proportion of $0.097\%$, exhibits the smallest performance gaps of $0.32\%$ and $1.46\%$ on CIFAR-10 and CIFAR-100, respectively. Apart from network architecture, the dataset exerts a non-negligible influence, \emph{i.e.}, CIFAR-100, with a greater number of classes, consistently yields a larger accuracy gap than CIFAR-10, which has fewer classes. This phenomenon is observed in all experiments.

%------------------------------------------------------------------------

\subsection{Defense Performance}
\label{sec:dp}

To assess the defensive efficacy against gradient leakage attacks, three state-of-the-art attack methods were employed: \ac{DLG}~\cite{zhu2019deep}, \ac{GRNN}~\cite{ren2022grnn}, \ac{IG}~\cite{geiping2020inverting} and \ac{GGL}~\cite{li2022auditing}. A batch size of 1 and resolution of $32*32$ were used for MNIST, CIFAR-10, and CIFAR-100. ILSVRC2012 utilized a resolution of $256*256$ instead of $224*224$, as \ac{GRNN} is only capable of generating images with resolutions that are exponential multiples of $2$. Images with lower resolution were upsampled using linear interpolation.

In Table~\ref{tab:dp}, qualitative results are presented to demonstrate the comparative performance of defense efficacy against targeted leakage attacks, including \ac{DLG}, \ac{GRNN} and \ac{IG}, with various backbone networks and benchmark datasets. \ac{DLG} consistently fails when employing \emph{ResNet-20} and \emph{ResNet-18}, and as a result, no data is presented. In the case of utilizing \emph{LeNet}, the reconstructed image is entirely perturbed once the key-lock module is integrated into the network. When a \ac{BN} layer or key-lock module was incorporated into \emph{LeNet}, \ac{DLG} failed to reconstruct accurate images in all experiments, as \ac{DLG} directly regresses image pixels by approximating the gradient. In other words, \ac{DLG} can only reconstruct the true image from a gradient with explicit input data information (refer to Section~\ref{sec:TAGL} Claim~\ref{claim:1}: Eqn.~\ref{equ:gradient_loss_linear_function}, Claim~\ref{claim:2}: Eqn.~\ref{equ:gradient_loss_fc_weight}, and Claim~\ref{claim:3}: Eqn.~\ref{equ:gradient_loss_theta_prime}). However, both the \ac{BN} layer and key-lock module re-normalize feature maps, leading to latent space misalignment and input data information ambiguity within the gradient (see Section~\ref{sec:TAGL} Claim~\ref{claim:4}: Eqn.~\ref{equ:gradient_loss_un} and Proposition~\ref{proposition:2} Eqn.~\ref{equ:gradient_loss_hat_uu}). According to the quantitative results on gradient leakage attacks in Table~\ref{tab:qc}, \ac{GRNN} is more potent than \ac{DLG} and \ac{IG}. Nevertheless, the proposed key-lock module can still effectively prevent true image leakage from gradient-based reconstruction. When employing GRNN on \emph{ResNet-18} with an image resolution of $256*256$, the reconstructed image is not recognizable for private data identification, even though the reconstructed image contains minimal, albeit visible, information from the true image.

\begin{table}[!ht]
\setlength{\tabcolsep}{3pt}
\begin{center}
\caption{Typical experimental results performed on \ac{GGL} with and without our proposed FedKL are shown below. The backbone network is ResNet-18 and the dataset is ILSVRC2012 with a resolution of $256*256$.}
\label{tab:ggl}
\begin{tabular}{c c | c c}
&\textbf{True} & \textbf{w/o KL} & \textbf{w/ KL} \\
\multicolumn{1}{m{0.1cm}}{\rotatebox{90}{\textbf{black grouse}}} & \makecell*[c]{\includegraphics[width=0.26\linewidth]{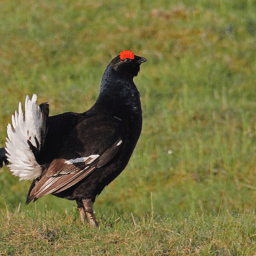}} & \makecell*[c]{\includegraphics[width=0.26\linewidth]{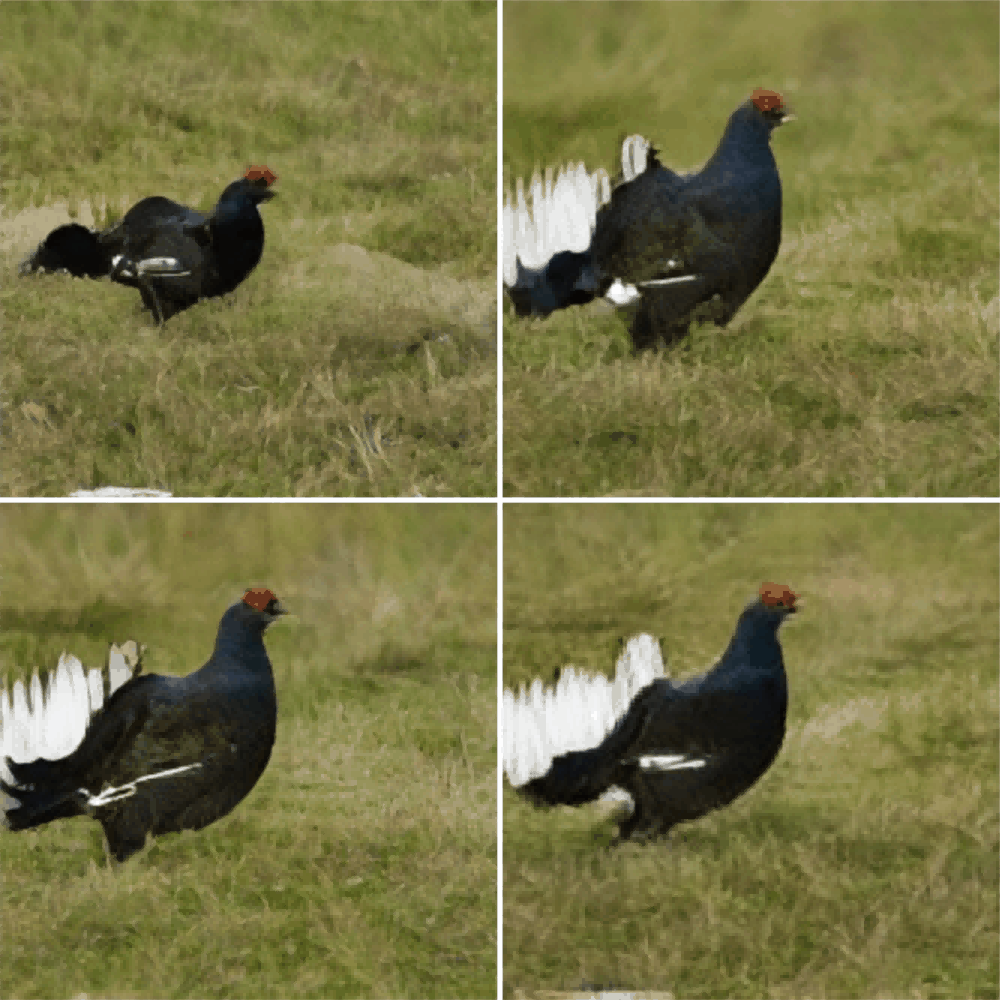}} & \makecell*[c]{\includegraphics[width=0.26\linewidth]{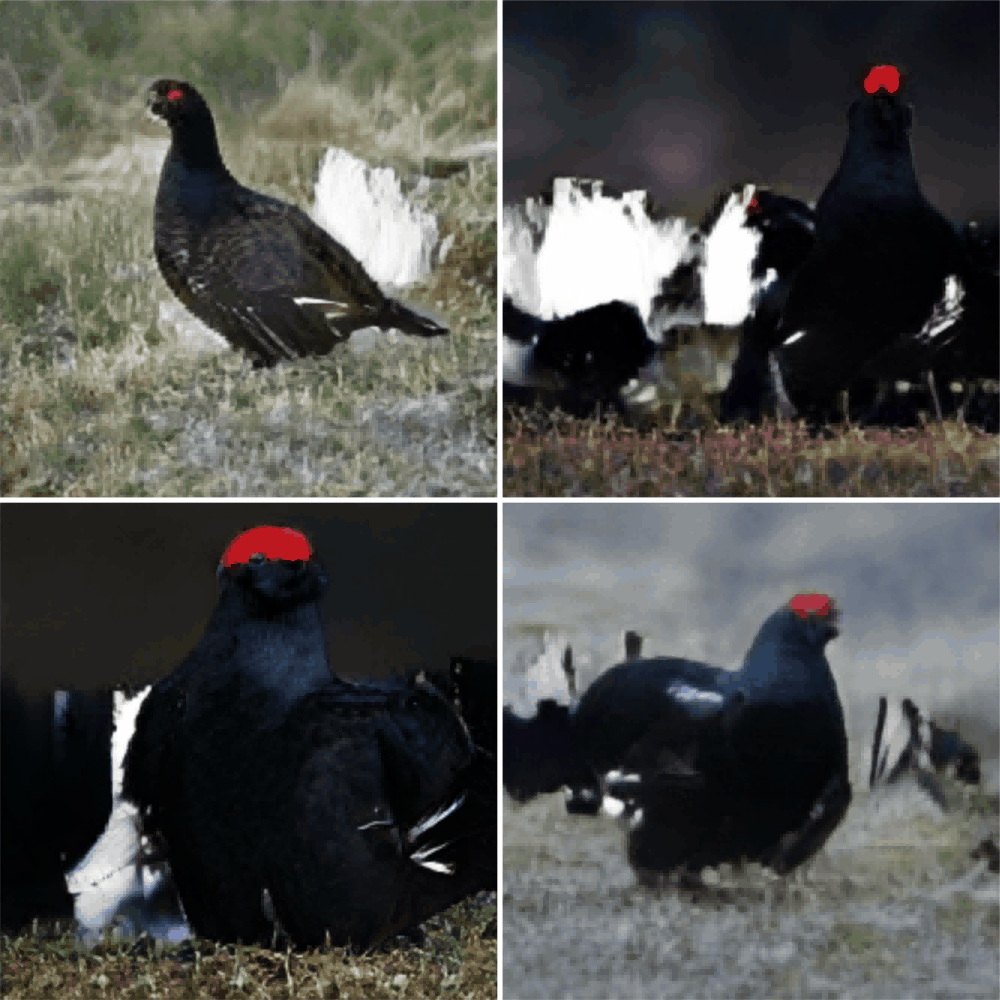}} \\

\multicolumn{1}{m{0.1cm}}{\rotatebox{90}{\textbf{tiger beetle}}} & \makecell*[c]{\includegraphics[width=0.26\linewidth]{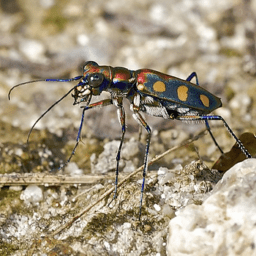}} & \makecell*[c]{\includegraphics[width=0.26\linewidth]{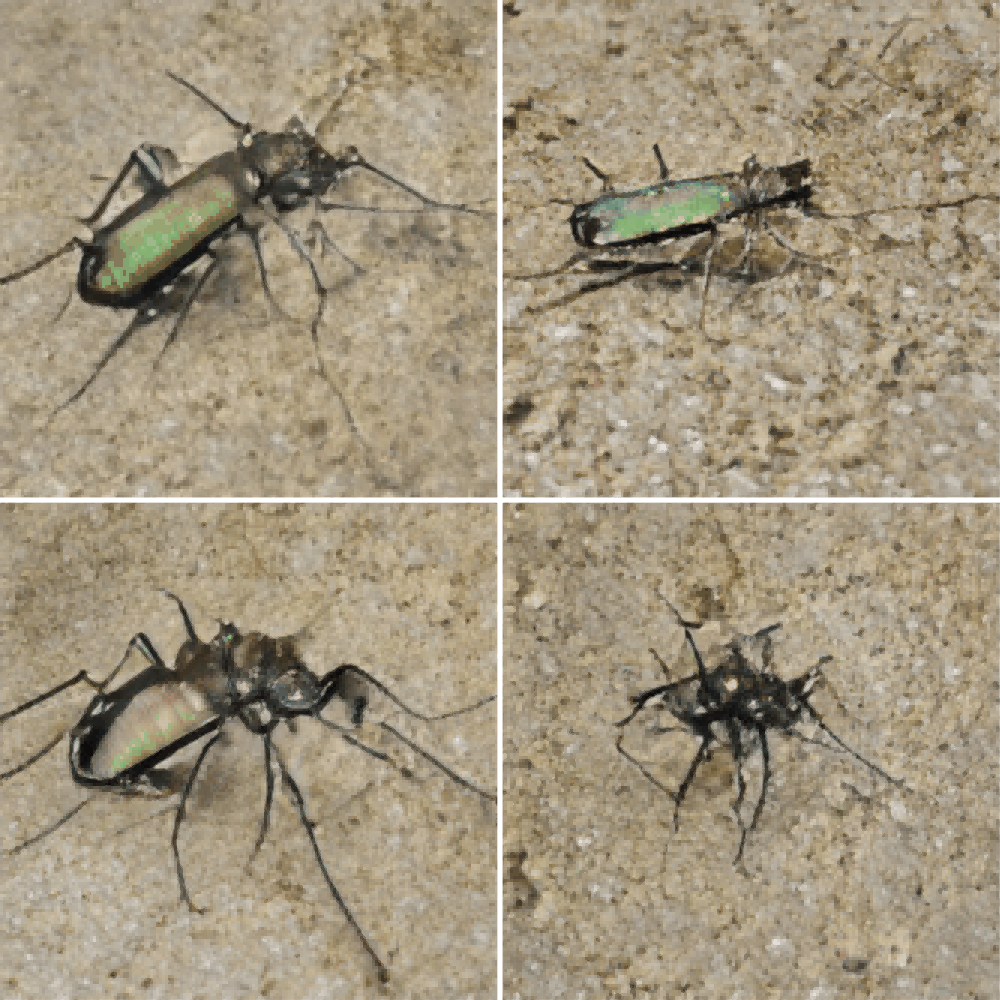}} & \makecell*[c]{\includegraphics[width=0.26\linewidth]{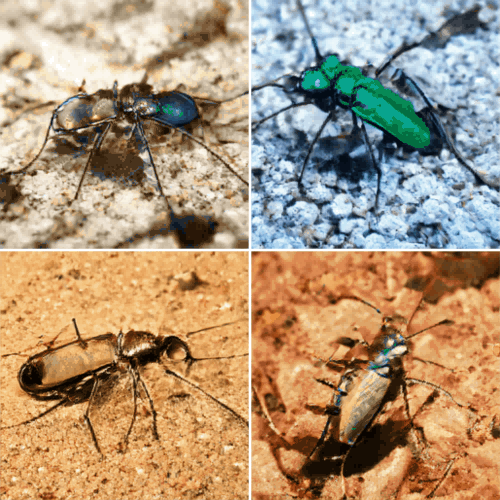}} \\

\multicolumn{1}{m{0.1cm}}{\rotatebox{90}{\textbf{cliff dwelling}}} & \makecell*[c]{\includegraphics[width=0.26\linewidth]{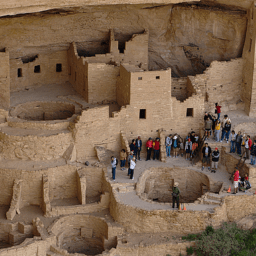}} & \makecell*[c]{\includegraphics[width=0.26\linewidth]{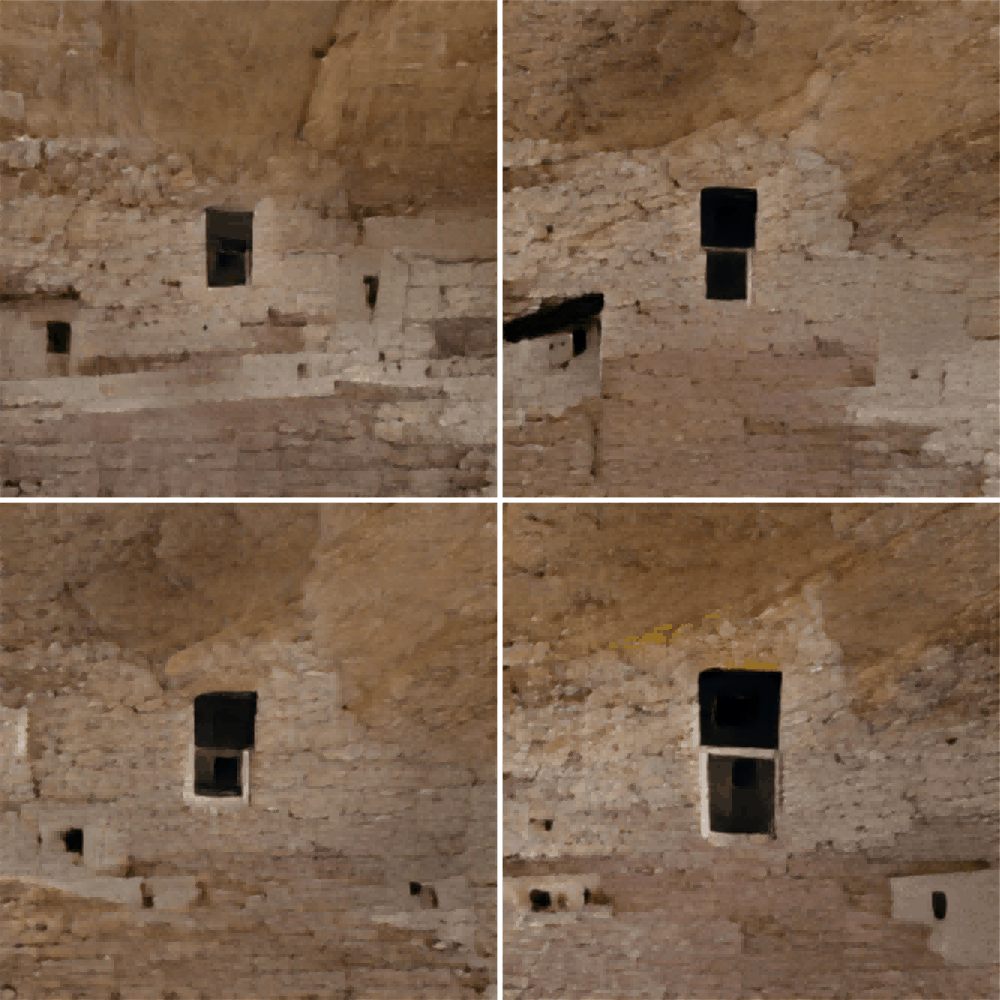}} & \makecell*[c]{\includegraphics[width=0.26\linewidth]{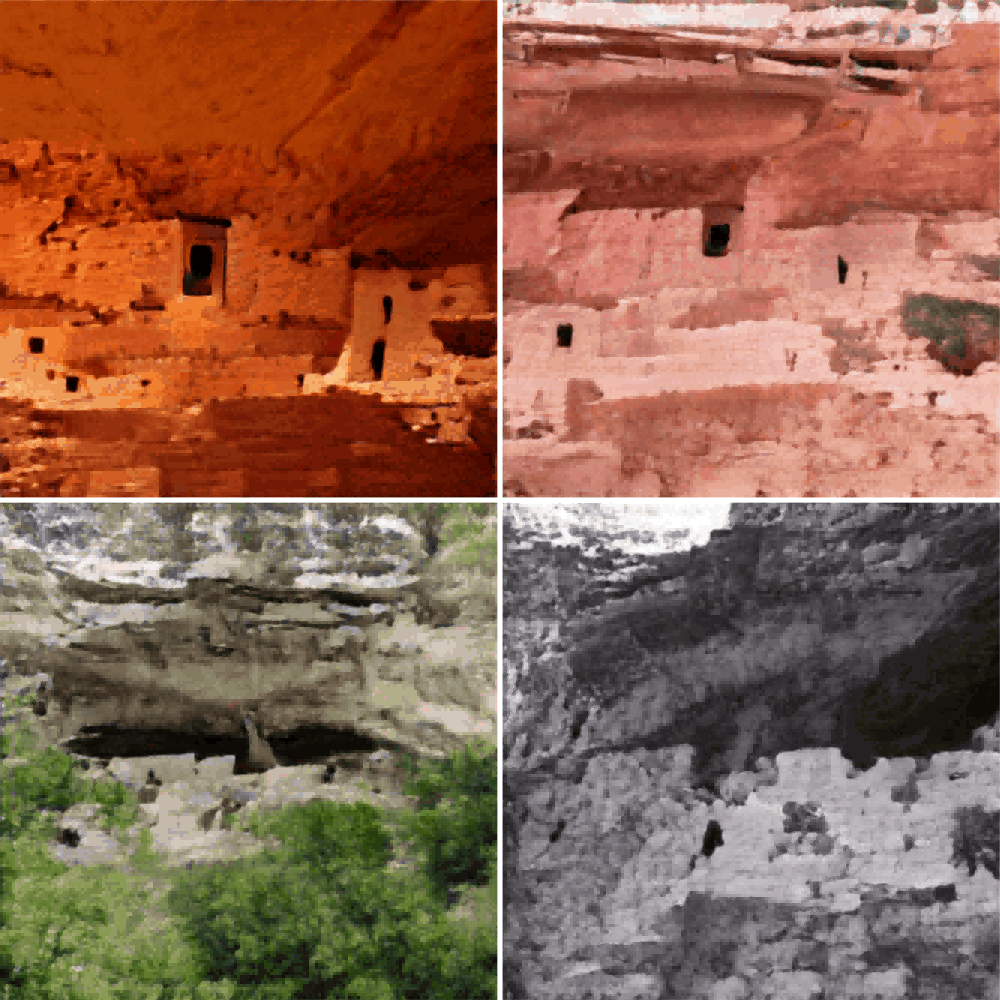}} \\

\multicolumn{1}{m{0.1cm}}{\rotatebox{90}{\textbf{basset hound}}} & \makecell*[c]{\includegraphics[width=0.26\linewidth]{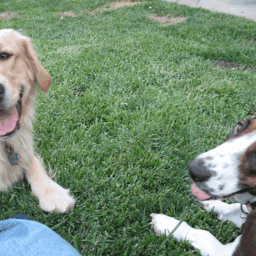}} & \makecell*[c]{\includegraphics[width=0.26\linewidth]{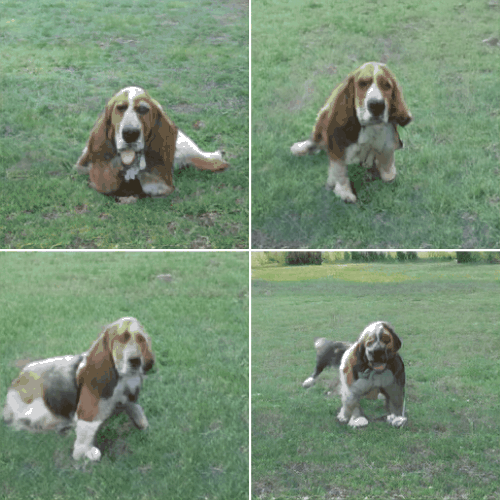}} & \makecell*[c]{\includegraphics[width=0.26\linewidth]{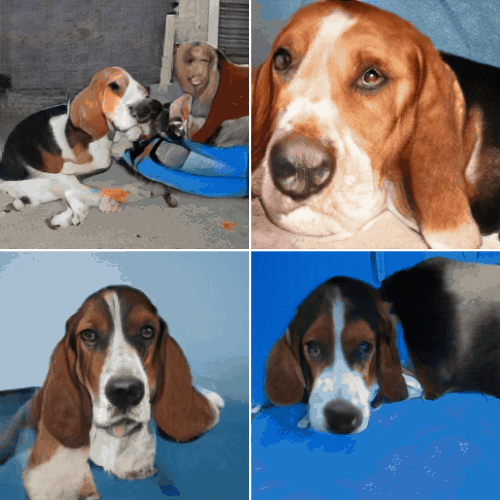}} \\

\multicolumn{1}{m{0.1cm}}{\rotatebox{90}{\textbf{sweatshirt}}} & \makecell*[c]{\includegraphics[width=0.26\linewidth]{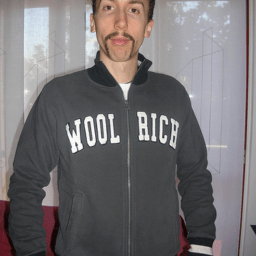}} & \makecell*[c]{\includegraphics[width=0.26\linewidth]{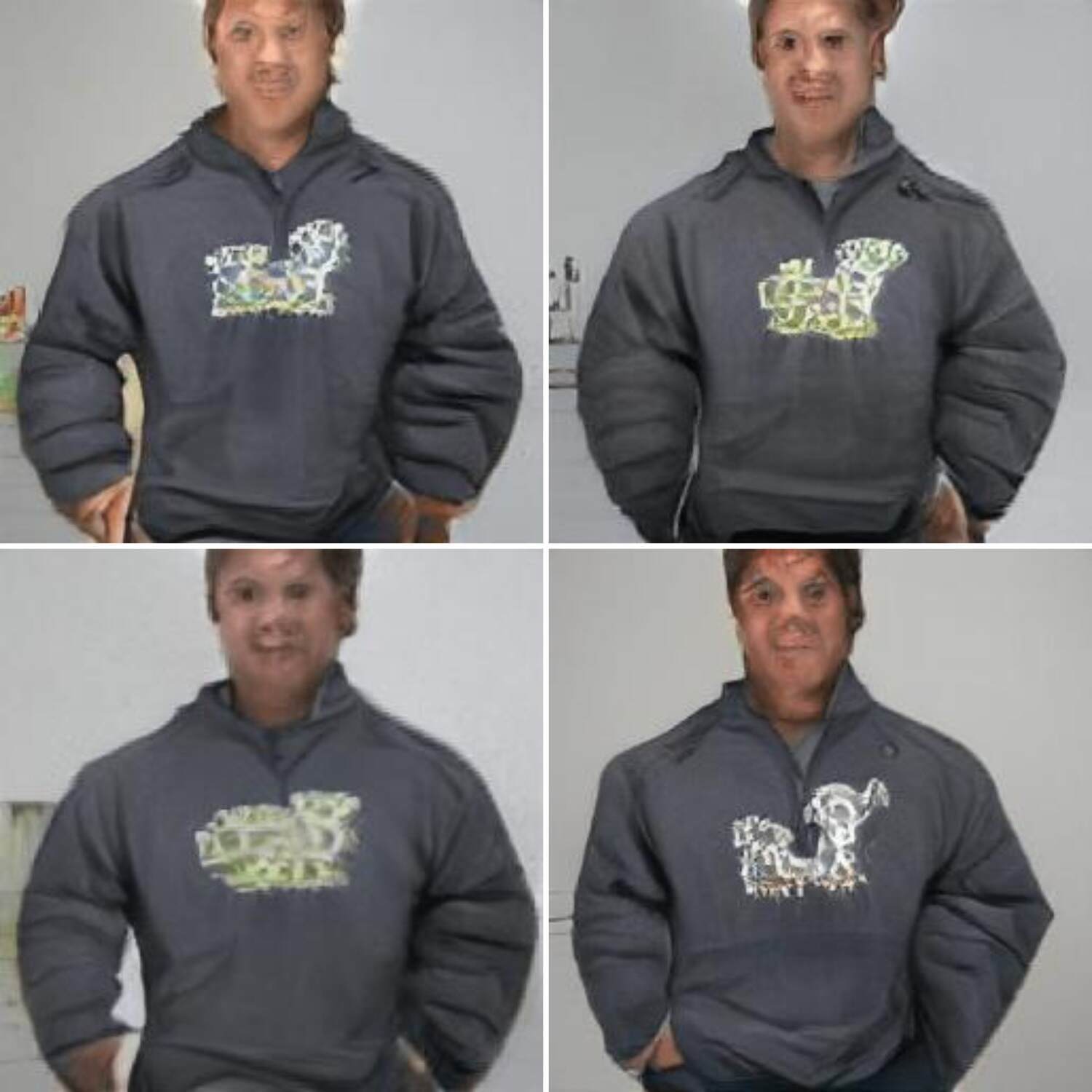}} & \makecell*[c]{\includegraphics[width=0.26\linewidth]{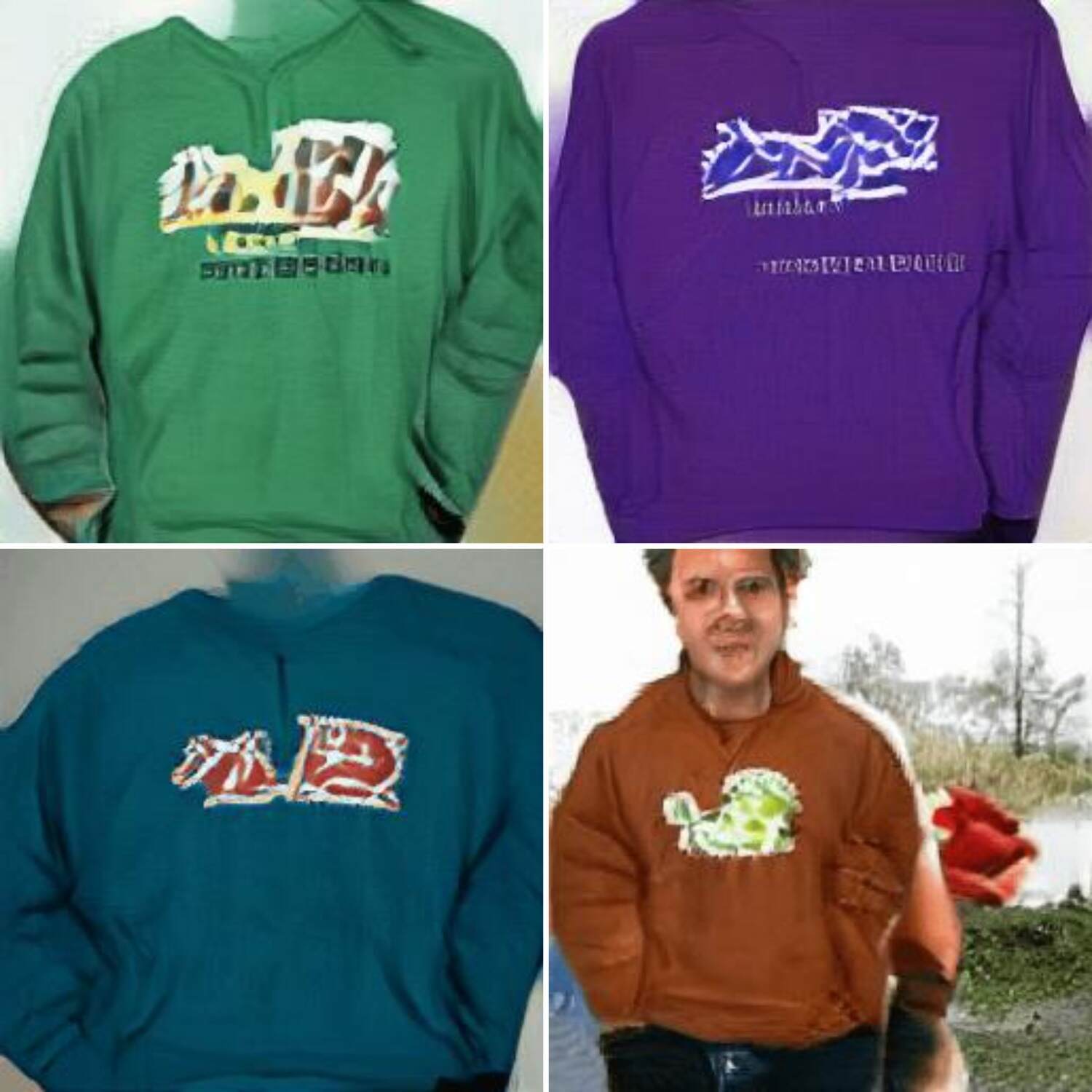}} \\

\multicolumn{1}{m{0.1cm}}{\rotatebox{90}{\textbf{radiator grille}}} & \makecell*[c]{\includegraphics[width=0.26\linewidth]{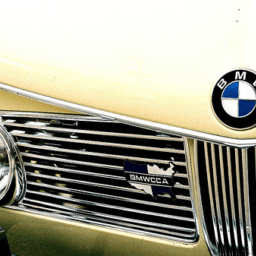}} & \makecell*[c]{\includegraphics[width=0.26\linewidth]{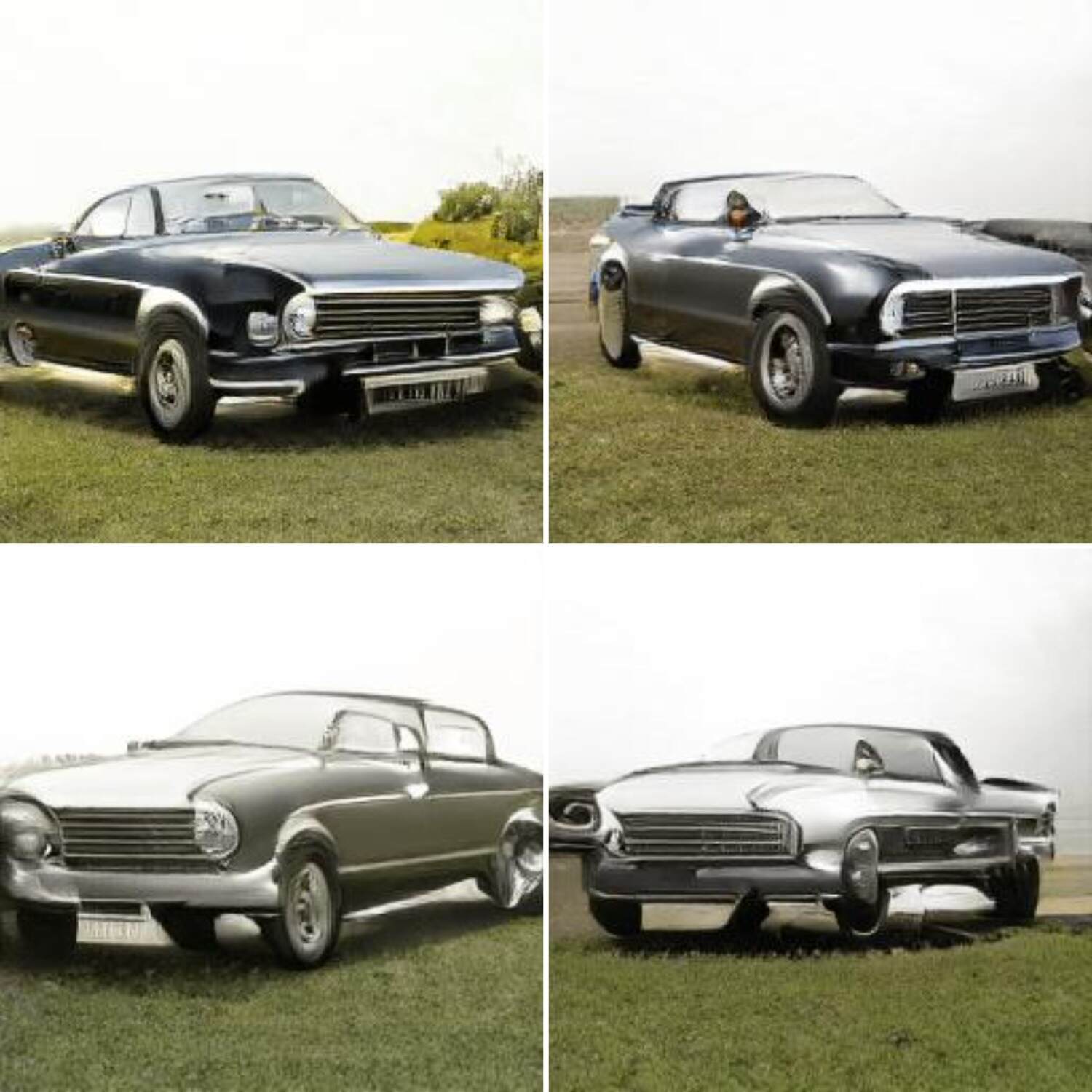}} & \makecell*[c]{\includegraphics[width=0.26\linewidth]{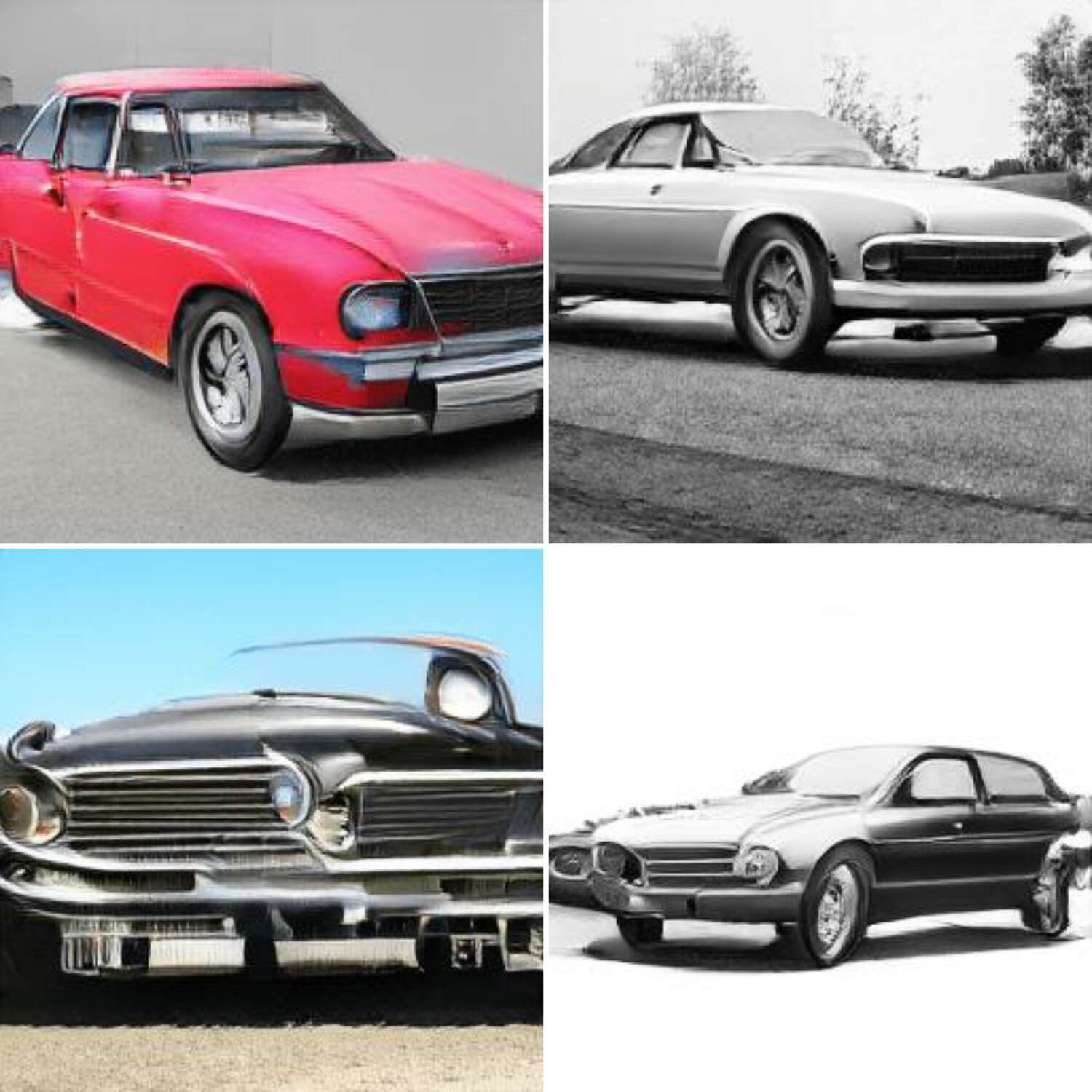}} \\

\multicolumn{1}{m{0.1cm}}{\rotatebox{90}{\textbf{pig}}} & \makecell*[c]{\includegraphics[width=0.26\linewidth]{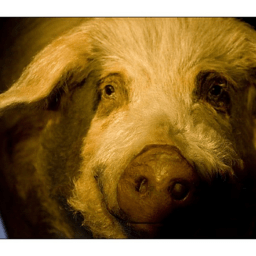}} & \makecell*[c]{\includegraphics[width=0.26\linewidth]{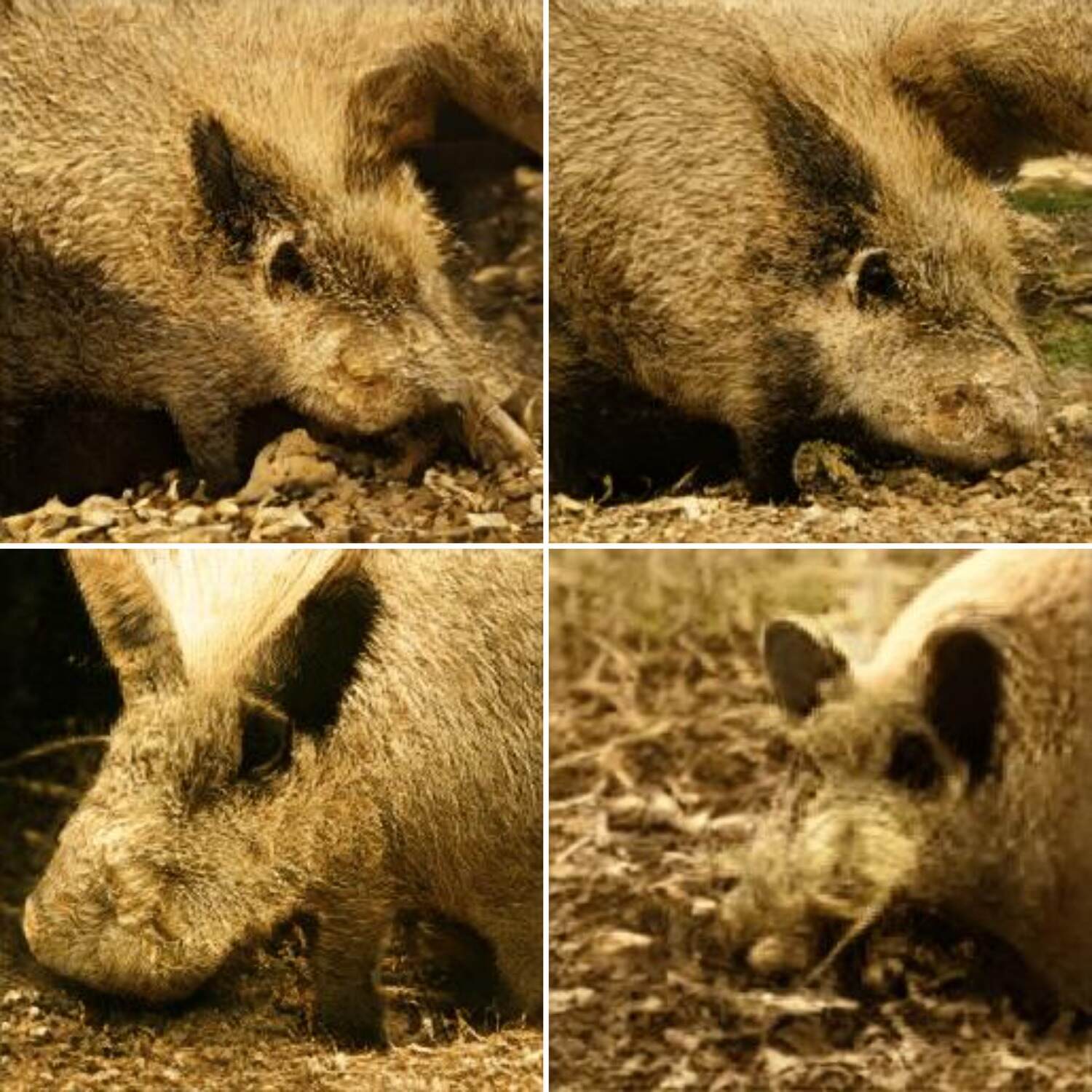}} & \makecell*[c]{\includegraphics[width=0.26\linewidth]{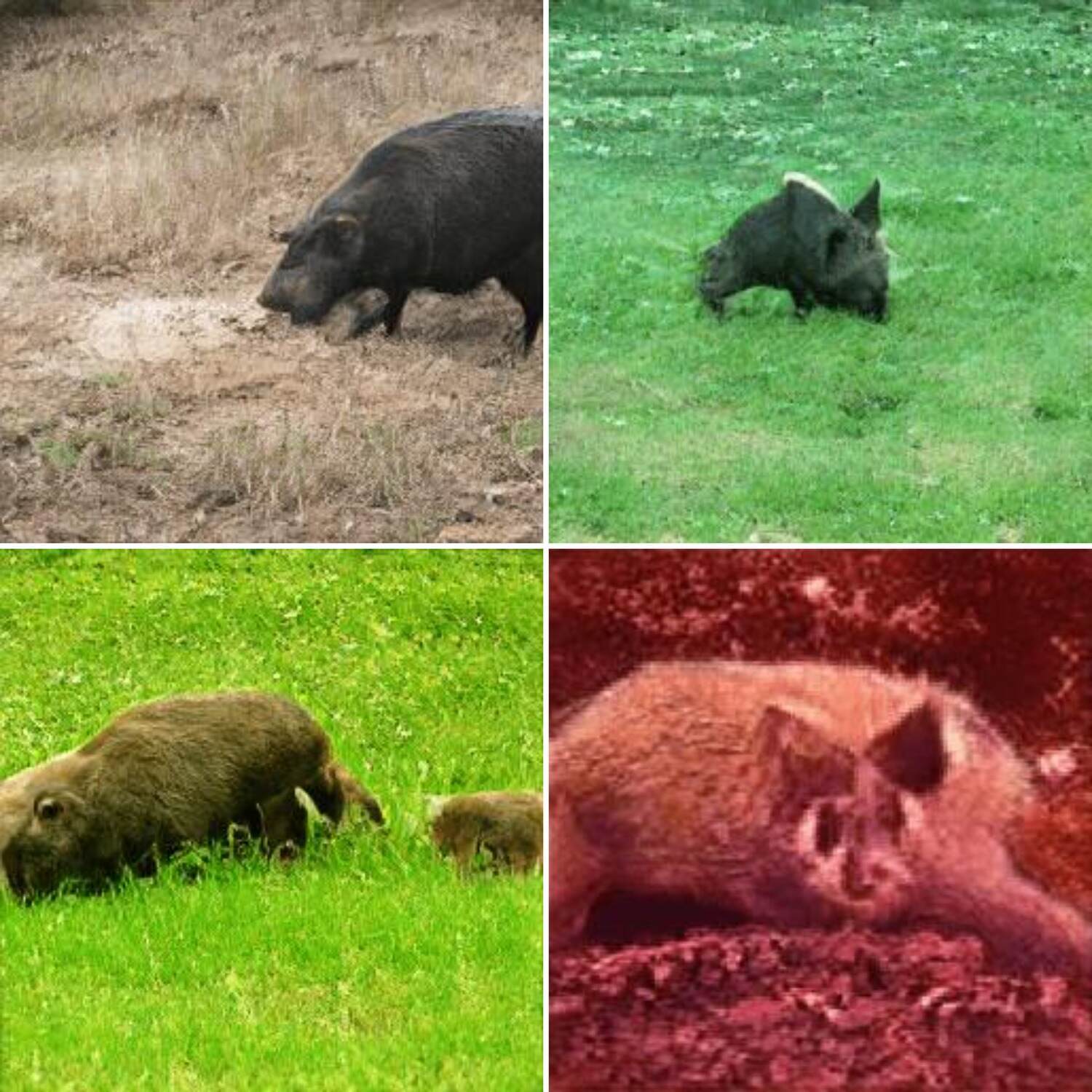}} \\

\multicolumn{1}{m{0.1cm}}{\rotatebox{90}{\textbf{Bedlington terrier}}} & \makecell*[c]{\includegraphics[width=0.26\linewidth]{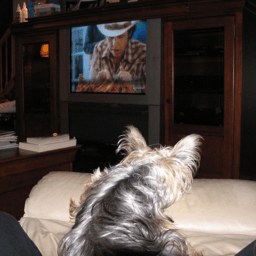}} & \makecell*[c]{\includegraphics[width=0.26\linewidth]{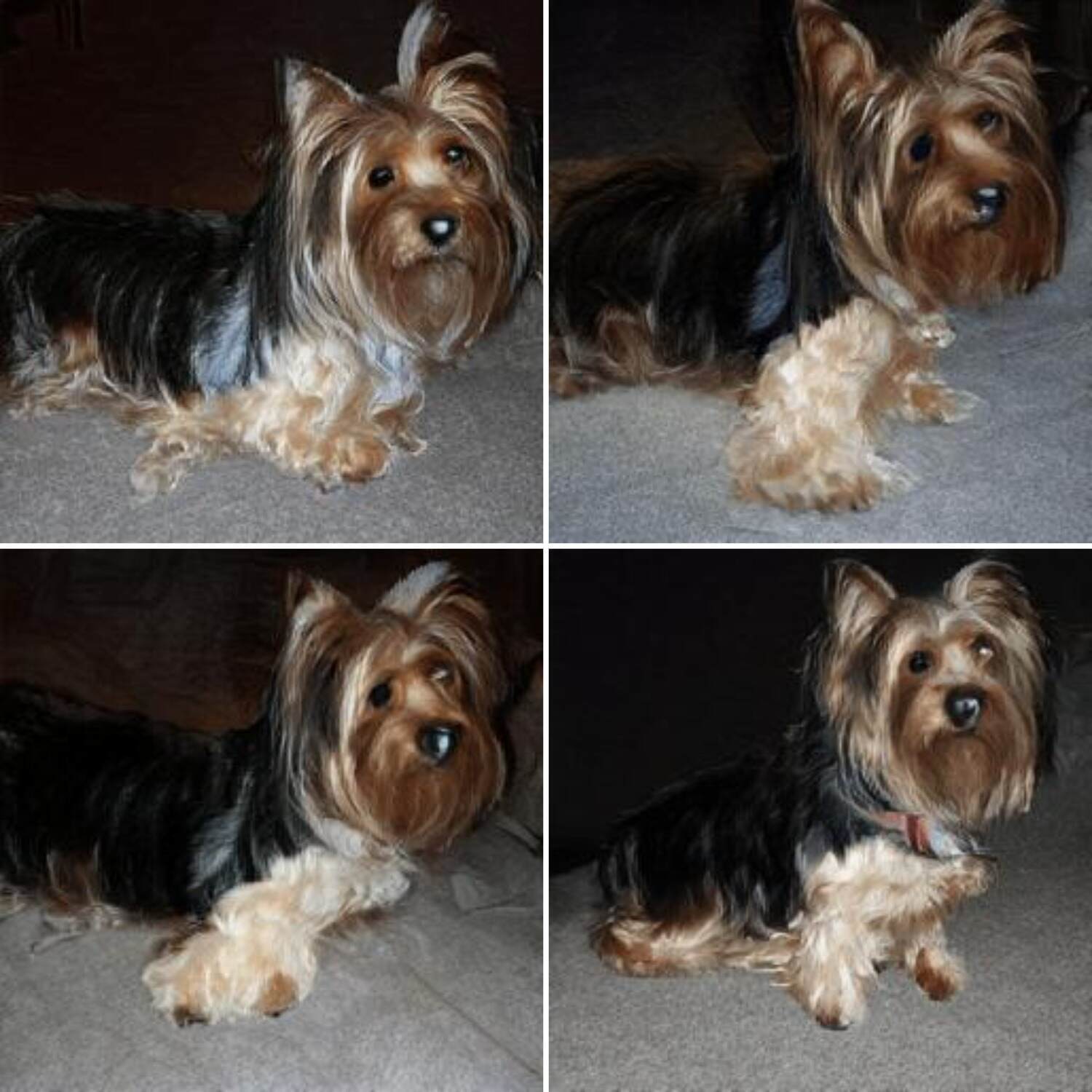}} & \makecell*[c]{\includegraphics[width=0.26\linewidth]{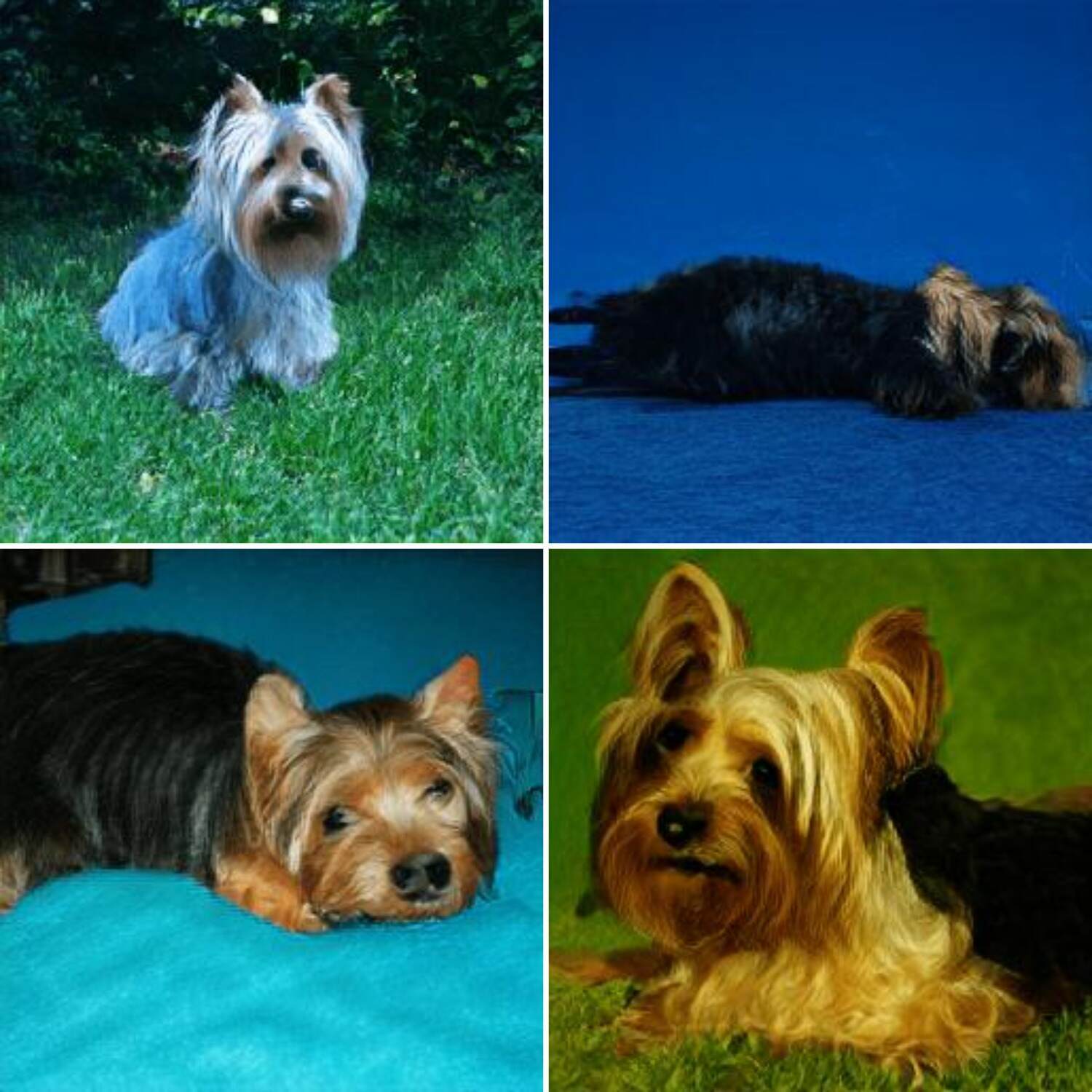}}
\end{tabular}
\end{center}
\end{table}

The \ac{GGL} first generates an image using the inferred ground-truth label to specify the image category, then fine-tunes the generated image based on the gradient information to make it as similar as possible to the true image. In Table~\ref{tab:ggl}, the model with the key-lock module generates samples with the same category as the true images. However, it can be observed that the appearances and spatial distributions are significantly different from the true image (see the w/KL column in Table~\ref{tab:ggl}). The key-lock module effectively prevents \ac{GGL} from inheriting the semantic and structural information of the true image. For instance, in the first row of the table, the samples generated from the \ac{GGL} model without the key-lock module display a black grouse with a similar green background color and orientation as the true input image. When the key-lock module is embedded into the model, \ac{GGL} can still generate images of black grouse using the same label. However, the color distribution and object orientation no longer match the true image. This is also true for the ``sweatshirt'' examples, where the model with the key-lock module generates four different colors of sweatshirts, none of which match the color (gray) in the true image. In terms of spatial information, significant variance can be observed across all examples when the proposed key-lock module is used. For instance, in the last column of Table~\ref{tab:ggl}, the generated images of ``basset hound'' exhibit various camera perspectives and distances of view fields. Consequently, it can be concluded that the generated images with the key-lock module are semantically different from the true image, proving that the true image information is effectively protected from the gradient.

\begin{table*}[!ht]
\setlength{\tabcolsep}{8pt}
\begin{center}
\caption{Quantitative comparison of \ac{DLG}, \ac{GRNN} and \ac{GGL} with/without key-lock module. The results are computed from the generated images and their associated true images.}
\label{tab:qc}
\begin{tabular}{c c c c >{\columncolor{blue!10}}c c >{\columncolor{red!10}}c c >{\columncolor{blue!10}}c c >{\columncolor{red!10}}c c >{\columncolor{blue!10}}c}
\hline
\multirow{2}{*}{\textbf{Method}} & \multirow{2}{*}{\textbf{Model}} & \multirow{2}{*}{\textbf{Dataset}} & \multicolumn{2}{c}{MSE$\downarrow$} & \multicolumn{2}{c}{PSNR$\uparrow$} & \multicolumn{2}{c}{LPIPS-V$\downarrow$} & \multicolumn{2}{c}{LPIPS-A$\downarrow$} & \multicolumn{2}{c}{SSIM$\uparrow$} \\
\cmidrule(r){4-5}\cmidrule(r){6-7}\cmidrule(r){8-9}\cmidrule(r){10-11}\cmidrule(r){12-13}
&&& w/o & w/ & w/o & w/ & w/o & w/ & w/o & w/ & w/o & w/ \\
\hline
\multirow{3}{*}{\textbf{DLG}} & \multirow{3}{*}{\makecell[c]{\emph{LeNet}\\(32*32)}}
& MNIST & 49.97 & 140.87 & 32.71 & 26.64 & 0.27 & 0.70 & 0.13 & 0.56 & 0.724 & -0.003 \\
&& C-10 & 29.62 & 98.76 & 41.14 & 28.20 & 0.13 & 0.48 & 0.08 & 0.32 & 0.724 & 0.006\\
&& C-100 & 25.30 & 101.44 & 39.85 & 28.09 & 0.08 & 0.49 & 0.04 & 0.33 & 0.859 & 0.009 \\
\hline
\multirow{9}{*}{\textbf{GRNN}} & \multirow{3}{*}{\makecell[c]{\emph{LeNet}\\(32*32)}}
& MNIST & 0.35 & 103.46 & 52.78 & 28.00 & 0.00 & 0.69 & 0.00 & 0.45 & 1.000 & 0.108 \\
&& C-10 & 0.78 & 92.12 & 49.29 & 28.52 & 0.00 & 0.49 & 0.00 & 0.26 & 1.000 & 0.101 \\
&& C-100 & 1.05 & 100.82 & 48.65 & 28.18 & 0.00 & 0.49 & 0.00 & 0.28 & 0.999 & 0.085 \\
\cline{3-13}
& \multirow{3}{*}{\makecell[c]{\emph{ResNet-20}\\(32*32)}}
& MNIST & 1.62 & 101.28 & 47.82 & 28.14 & 0.04 & 0.64 & 0.00 & 0.36 & 0.998 & 0.010\\
&& C-10 & 9.12 & 86.42 & 40.86 & 28.85 & 0.00 & 0.43 & 0.00 & 0.24 & 0.984 & 0.051\\
&& C-100 & 20.79 & 99.25 & 38.09 & 28.28 & 0.00 & 0.45 & 0.00 & 0.25 & 0.963 & 0.032 \\
\cline{3-13}
& \multirow{3}{*}{\makecell[c]{\emph{ResNet-18}\\(256*256)}}
& C-10 & 15.59 & 72.55 & 38.41 & 29.56 & 0.01 & 0.46 & 0.01 & 0.54 & 0.965 & 0.153 \\
&& C-100 & 34.63 & 85.33 & 34.13 & 28.97 & 0.03 & 0.48 & 0.02 & 0.55 & 0.917 & 0.159 \\
&& ILSVRC & 13.26 & 62.87 & 38.36 & 30.22 & 0.04 & 0.40 & 0.03 & 0.48 & 0.932 & 0.273 \\
\hline
\multirow{9}{*}{\textbf{IG}} & \multirow{3}{*}{\makecell[c]{\emph{LeNet}\\(32*32)}}
& MNIST & 15.59 & 48.66 & 36.22 & 31.27 & 0.05 & 0.20 & 0.02 & 0.23 & 0.890 & 0.459 \\
&& C-10 & 15.61 & 48.59 & 36.46 & 31.54 & 0.08 & 0.27 & 0.04 & 0.19 & 0.735 & 0.349 \\
&& C-100 & 15.80 & 49.72 & 36.35 & 31.38 & 0.08 & 0.24 & 0.04 & 0.20 & 0.776 & 0.386 \\
\cline{3-13}
& \multirow{3}{*}{\makecell[c]{\emph{ResNet-20}\\(32*32)}}
& MNIST & 51.20 & 58.90 & 31.06 & 30.44 & 0.24 & 0.28 & 0.20 & 0.19 & 0.262 & 0.253\\
&& C-10 & 34.60 & 45.06 & 33.07 & 31.74 & 0.19 & 0.29 & 0.11 & 0.16 & 0.298 & 0.150\\
&& C-100 & 40.84 & 45.07 & 32.15 & 31.70 & 0.23 & 0.28 & 0.14 & 0.15 & 0.238 & 0.187 \\
\cline{3-13}
& \multirow{3}{*}{\makecell[c]{\emph{ResNet-18}\\(256*256)}}
& C-10 & 34.81 & 53.15 & 33.20 & 30.92 & 0.27 & 0.34 & 0.22 & 0.32 & 0.497 & 0.635 \\
&& C-100 & 42.97 & 55.16 & 32.30 & 30.77 & 0.31 & 0.33 & 0.25 & 0.34 & 0.422 & 0.608 \\
&& ILSVRC & 61.74 & 86.10 & 30.32 & 28.81 & 0.43 & 0.57 & 0.37 & 0.58 & 0.170 & 0.258 \\
\hline
\textbf{GGL} & \makecell[c]{\emph{ResNet-18}\\(256*256)} & ILSVRC & 63.59 & 87.22 & 30.22 & 28.79 & 0.40 & 0.47 & 0.39 & 0.46 & 0.231 & 0.180 \\
\hline
\end{tabular}
\end{center}
% \vspace{-8pt}
\end{table*}

In order to quantitatively evaluate the proposed key-lock module against state-of-the-art gradient leakage methods, we employed four evaluation metrics, including \ac{MSE}, \ac{PSNR}, \ac{LPIPS} (using \emph{VGGNet} and \emph{AlexNet}), and \ac{SSIM}. In Table~\ref{tab:qc}, the results were calculated from the generated images and their corresponding true images. It is consistent that the generated images from the model with the proposed key-lock module have significantly lower similarity compared to those from the model without the key-lock module. This indicates that the key-lock module can effectively defend against gradient leakage attacks. For instance, \ac{DLG} with \emph{LeNet} on the MNIST dataset has an \ac{MSE} score of $49.97$, while the one with the key-lock module achieves $140.87$, which is $181.91\%$ higher. When \emph{LeNet} is embedded with a key-lock module, the \ac{MSE} score of \ac{GRNN} on the MNIST dataset increases by $29,460\%$. As for \ac{PSNR}, most models without the key-lock module scored between $35$ and $50$. However, when the key-lock module is integrated into the model, the \ac{PSNR} scores drop to approximately $28$. In terms of \ac{LPIPS} and \ac{SSIM}, the former metric focuses on the semantic similarity between two images, while the latter one emphasizes structural similarity. The results from all the \ac{DLG}, \ac{GRNN} and \ac{IG} perform significantly worse when the key-lock module is embedded into the model. Furthermore, unlike \ac{DLG}, \ac{GRNN} and \ac{IG}, which generate original true images from the gradient directly, the images generated by \ac{GGL} are similar to the true images in terms of semantics and structure, as shown in Table~\ref{tab:ggl}. We found that \ac{GGL} with the key-lock module can still generate images with rich semantic and structural information. Nevertheless, the key-lock module can reduce the reconstruction capability of \ac{GGL} by preventing the true image's information from leaking through the gradient. As a result, the \ac{GGL} without the key-lock module obtains $0.40$ and $0.39$ on \ac{LPIPS} using \emph{VGGNet} and \emph{AlexNet}, respectively, whereas the results increase to $0.47$ and $0.46$ when the key-lock module is present. Although the \ac{SSIM} score is only $0.231$ for standard \ac{GGL}, it is reduced to $0.180$ by the key-lock module. Based on the quantitative results in Table~\ref{tab:qc}, we can conclude that the proposed key-lock module is capable of effectively protecting private information from being leaked through the gradient.

%----------------------------------------------------------------------
\begin{figure*}[ht!]
\centering
    \subfigure{
    \centering
    \includegraphics[width=0.45\linewidth]{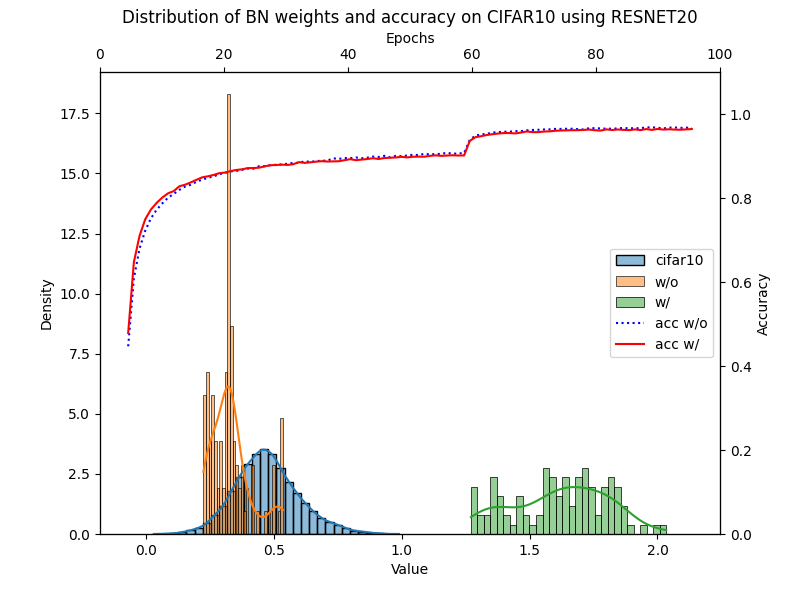}
    }
    \subfigure{
    \centering
    \includegraphics[width=0.45\linewidth]{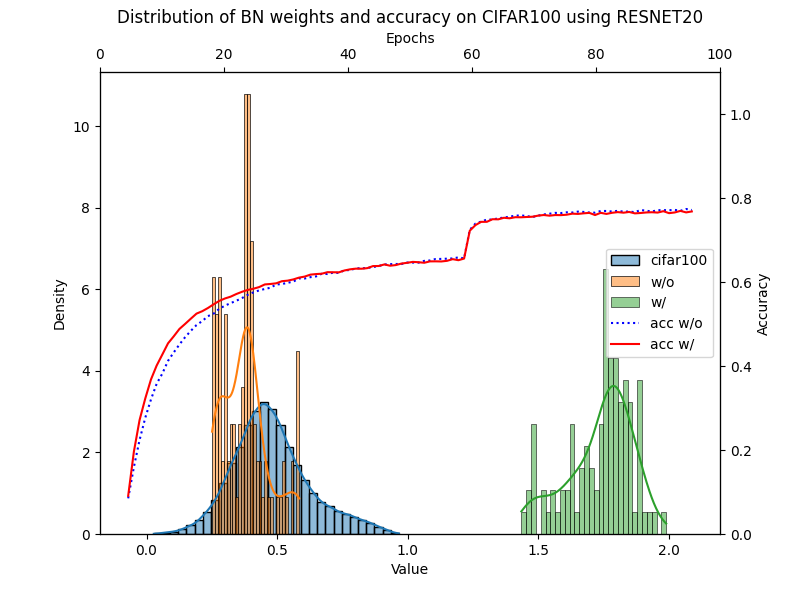}
    }
    % \vspace{-10pt}

    \subfigure{
    \centering
    \includegraphics[width=0.45\linewidth]{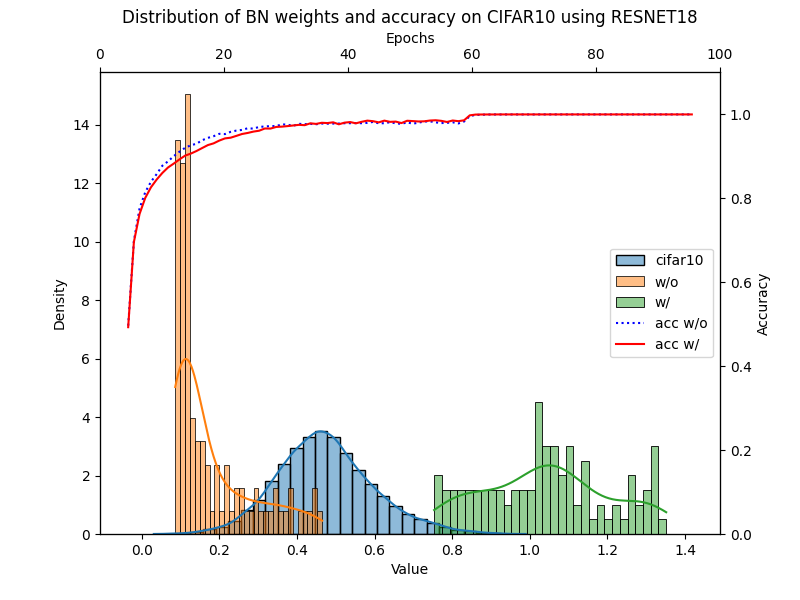}
    }
    \subfigure{
    \centering
    \includegraphics[width=0.45\linewidth]{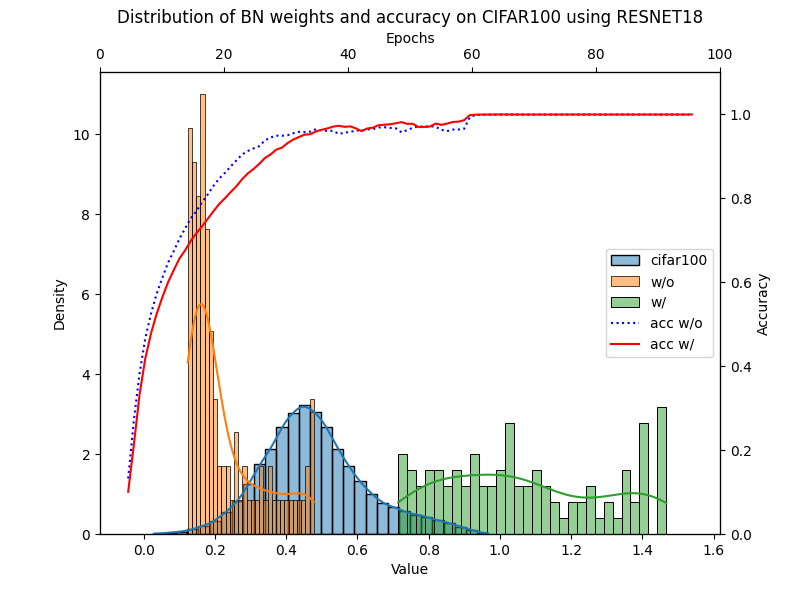}
    }
\caption{Distribution of dataset samples after applying BN scale and shift parameters, along with test accuracy on CIFAR-10/100 using ResNet-20 and ResNet-18. The key-lock module (green) alters the distributions compared to standard training (orange), while maintaining or improving accuracy (red vs. blue lines). Reference distributions (blue) correspond to the original dataset.}
\label{fig:bnd}
\end{figure*}

\subsection{Distributional Characteristics for BN Parameters}
\label{sec:dc}
To examine the representational impact of the key-lock module, we analyze the output distributions produced by the \ac{BN} layer after applying the scale and shift parameters to the dataset. Rather than inspecting the parameter values directly, we apply the \ac{BN} transformation to the entire dataset and visualize the resulting outputs. This allows us to observe how the internal data representation is altered by the presence of the key-lock mechanism. 

Fig.~\ref{fig:bnd} presents a comparative view of these distributions for \emph{ResNet-20} and \emph{ResNet-18} architectures on the CIFAR-10 and CIFAR-100 datasets. The distributions are shown in three forms: reference distributions obtained from clean models (in blue), distributions (in orange) by applying scale and shift parameters from models trained without the key-lock module to the dataset, and those from models trained with the key-lock module (in green). In both architectures, the distributions from models trained without the key-lock module closely resemble the reference distributions, indicating that conventional training produces stable and consistent internal representations aligned with the data distribution. However, a notable deviation is observed when the key-lock module is introduced. In the key-lock-enabled models, the \ac{BN} outputs become significantly more dispersed and irregular. This shift is especially evident in the \emph{ResNet-20} experiments, where the green distributions exhibit greater asymmetry and deviation from the reference baseline. These distributional changes stem from the fact that scale and shift parameters in the key-lock framework are not learned directly from the training data but are instead generated via a transformation of a private key input. This decouples the \ac{BN} behavior from the dataset statistics and introduces randomness specific to each client. Despite these disruptions in internal representation, the overlaid accuracy curves in Fig.~\ref{fig:bnd} show that the models maintain competitive or improved performance. For example, in \emph{ResNet-18} on CIFAR-100, the model trained with the key-lock module consistently outperforms its baseline counterpart throughout the training process. This confirms that the key-lock module does not degrade, and may even enhance, the discriminative power of the network. The changes in distribution observed here have important implications for privacy. Gradient leakage attacks typically rely on the consistent alignment of internal feature statistics across training rounds. By distorting these distributions through a private-key-based transformation, the key-lock module weakens the correlation between the gradient and the original input, thus reducing the effectiveness of reconstruction attacks. Unlike traditional noise-injection or perturbation methods, this defense emerges naturally from the architectural design, without introducing additional randomness into the gradient itself. In conclusion, the experimental results in Fig.~\ref{fig:bnd} provide concrete evidence that the key-lock module alters \ac{BN}-induced data transformations in a meaningful way. These changes disrupt gradient predictability, enhance model privacy, and preserve classification performance, thereby reinforcing the theoretical claims regarding the module’s effectiveness in defending against gradient-based leakage.

%------------------------------------------------------------------------
\subsection{Sharing Strategies of the Key-Lock Module}
\label{sec:sdss}

\begin{table*}[!ht]
\setlength{\tabcolsep}{3pt}
\begin{center}
\caption{Comparison of image reconstruction using \ac{DLG}, \ac{GRNN} and \ac{GGL} with key-lock module. The first case is to share only the private key sequence with the server. Then ``Lock'' means not sharing the key sequence, but sharing the gradient of the lock layer. ``Both'' is to share both the key sequence and gradient of the lock layer.}
% 1: share_key
% 2: with_lock_layer
% 3: share_key-with_lock_layer
\label{tab:cdss}
\begin{tabular}{c c | c c c | c c c | c}
\hline
\multicolumn{2}{c|}{\textbf{Model}} & \multicolumn{3}{c|}{\makecell[c]{\emph{LeNet}\\(32*32)}} & \multicolumn{4}{c}{\makecell[c]{\emph{ResNet-18}\\(256*256)}} \\
\hline
\multicolumn{2}{c|}{\textbf{Method}} & \multicolumn{3}{c|}{GRNN} & \multicolumn{3}{c|}{GRNN} & GGL \\
\hline
& & MNIST & C-10 & C-100 & C-10 & C-100 & ILSVRC & ILSVRC \\
% \hline
\multirow{12}{*}{\makecell[c]{\textbf{Shared}\\\textbf{Info.}}} & \textbf{Key} & \makecell*[c]{\includegraphics[width=0.08\linewidth]{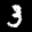}} & \makecell*[c]{\includegraphics[width=0.08\linewidth]{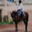}} & \makecell*[c]{\includegraphics[width=0.08\linewidth]{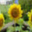}} & \makecell*[c]{\includegraphics[width=0.08\linewidth]{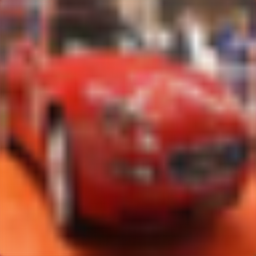}} & \makecell*[c]{\includegraphics[width=0.08\linewidth]{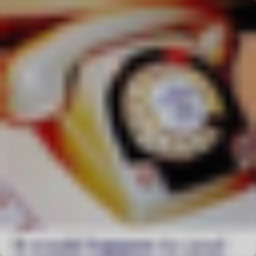}} & \makecell*[c]{\includegraphics[width=0.08\linewidth]{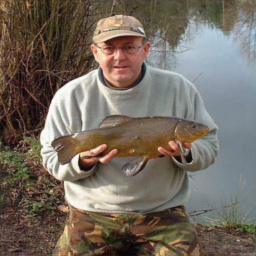}} & \makecell*[c]{\includegraphics[width=0.08\linewidth]{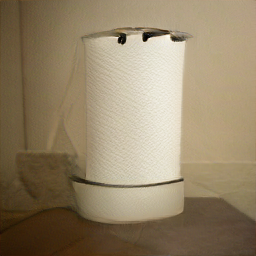}}\\

& \textbf{Lock} & \makecell*[c]{\includegraphics[width=0.08\linewidth]{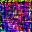}} & \makecell*[c]{\includegraphics[width=0.08\linewidth]{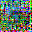}} & \makecell*[c]{\includegraphics[width=0.08\linewidth]{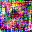}} & \makecell*[c]{\includegraphics[width=0.08\linewidth]{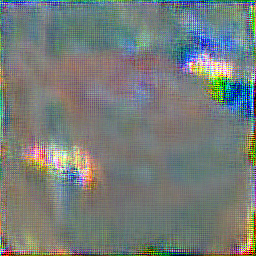}} & \makecell*[c]{\includegraphics[width=0.08\linewidth]{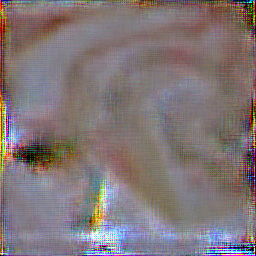}} & \makecell*[c]{\includegraphics[width=0.08\linewidth]{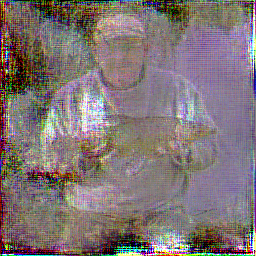}} & \makecell*[c]{\includegraphics[width=0.08\linewidth]{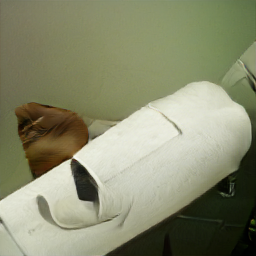}}\\

& \textbf{Both} & \makecell*[c]{\includegraphics[width=0.08\linewidth]{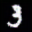}} & \makecell*[c]{\includegraphics[width=0.08\linewidth]{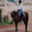}} & \makecell*[c]{\includegraphics[width=0.08\linewidth]{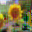}} & \makecell*[c]{\includegraphics[width=0.08\linewidth]{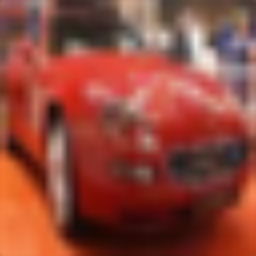}} & \makecell*[c]{\includegraphics[width=0.08\linewidth]{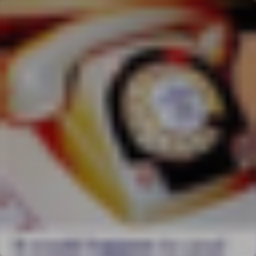}} & \makecell*[c]{\includegraphics[width=0.08\linewidth]{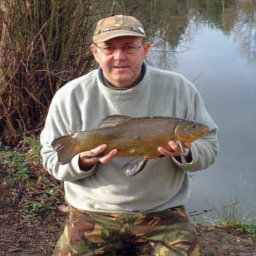}} & \makecell*[c]{\includegraphics[width=0.08\linewidth]{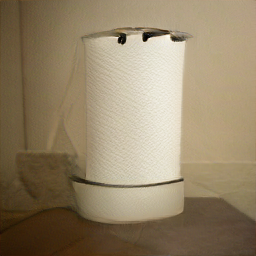}}\\

\hline
\multicolumn{2}{c|}{\textbf{True}} & \makecell*[c]{\includegraphics[width=0.08\linewidth]{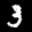}} & \makecell*[c]{\includegraphics[width=0.08\linewidth]{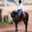}} & \makecell*[c]{\includegraphics[width=0.08\linewidth]{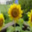}} & \makecell*[c]{\includegraphics[width=0.08\linewidth]{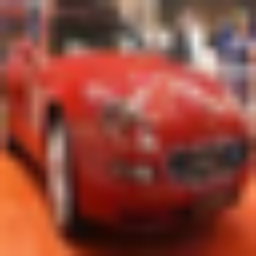}} & \makecell*[c]{\includegraphics[width=0.08\linewidth]{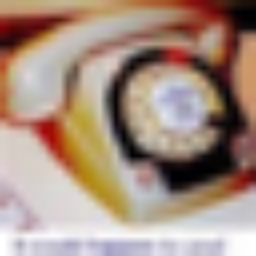}} & \makecell*[c]{\includegraphics[width=0.08\linewidth]{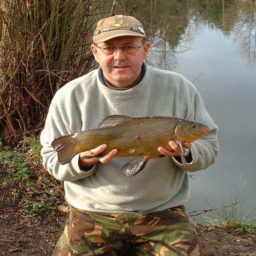}} & \makecell*[c]{\includegraphics[width=0.08\linewidth]{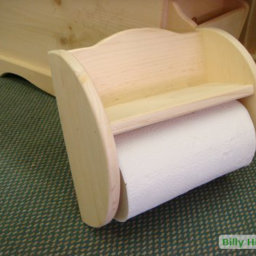}}\\
\hline
\end{tabular}
\end{center}
\end{table*}

Our proposed module comprises two components to defend against gradient leakage in gradient-sharing-based training systems, such as \ac{FL}. We theoretically demonstrated in Section~\ref{sec:TAGL} that the key-lock module can effectively prevent the inheritance of input data information from being embedded into the gradient of the model. In this section, we experimentally evaluate the influence of the key sequence and lock layer separately for gradient leakage defense. We employ three information-sharing strategies for \ac{GRNN} and \ac{GGL} with different backbones trained on various datasets. \ac{DLG} failed to recover any true images across all scenarios, and the potential reasons have already been discussed in Section~\ref{sec:dp}. What is more, \ac{IG} does not performance as good as \ac{GRNN}. As a result, no examples from \ac{DLG} and \ac{IG} are presented in Table~\ref{tab:cdss}. Factors such as network complexity and image resolution are considered. The private key can be used to calculate the gradient of the lock layer, but the reverse is not possible (see Eqns.~\ref{equ:26} \& \ref{equ:27}). Consequently, the gradient in Eqn.~\ref{equ:gradient_loss_hat_uu} can be inferred theoretically. When the image resolution is fixed at $32*32$, \ac{GRNN} can successfully perform gradient leakage attacks by providing a private key only, while it fails to recover any meaningful visual content by providing the parameters of the key-lock module only. We can conclude that although both the key and the parameters of the key-lock module contribute to gradient protection, the private key information plays a more significant role in defending against gradient leakage attacks. To mitigate any potential risk, we ensure that both the key and the lock layer's gradients are retained on the client side. This dual protection mechanism reinforces the security against gradient leakage.
%=====================================================================
\section{Conclusion}
\label{sec:cc}
In this paper, we provided comprehensive theoretical analysis on gradient leakage, based on which we proposed a gradient leakage defense method, called FedKL, for the \ac{FL} system. The proposed key-lock module prevents the input data information from leaking through the gradient during the training stage. Importantly, we theoretically demonstrated the efficacy of the key-lock module in defending against gradient leakage attacks. Extensive empirical studies were conducted with three state-of-the-art attack methods. All the quantitative and qualitative comparison results indicate that by incorporating the proposed key-lock module into the model, reconstructing input images from the publicly shared gradient is no longer feasible. Additionally, we discussed the influence of the key-lock module on the model in terms of classification accuracy. The implementation of our proposed FedKL is made publicly available to ensure consistent replication and facilitate further comparison for researchers in related fields.
%=====================================================================

\bibliographystyle{IEEEtran}
\bibliography{ref}

\end{document}